\documentclass{article}

\PassOptionsToPackage{round}{natbib}
\usepackage[preprint]{neurips_2026}
\usepackage[utf8]{inputenc}
\usepackage[T1]{fontenc}    
\usepackage{comment}
\usepackage{url}          
\usepackage{booktabs}       
\usepackage{amsfonts}      
\usepackage{nicefrac}       
\usepackage{microtype}      
\usepackage{xcolor}         
\usepackage{graphicx}
\usepackage{amsfonts}      
\usepackage{nicefrac}
\usepackage{amsthm}
\usepackage{amssymb}
\usepackage{mathrsfs}
\usepackage{amsfonts}
\usepackage{amsmath}
\usepackage{multirow}
\usepackage{mathrsfs}
\usepackage{dsfont}
\usepackage{subcaption}
\usepackage{wrapfig}
\usepackage{colortbl}
\usepackage{bm}
\usepackage{newunicodechar}
\newunicodechar{ℓ}{$\ell$}

\newtheorem{theorem}{Theorem}
\newtheorem{prop}{Proposition}
\newtheorem{lemma}{Lemma}

\newtheorem{definition}{Definition}

\definecolor{darkblue}{rgb}{0,0.08,0.45}

\usepackage[unicode=true,colorlinks=true,citecolor=darkblue,linkcolor=darkblue,urlcolor=darkblue]{hyperref}

\newcommand{\cF}{\mathcal{F}}

\newcommand{\cL}{\mathcal{L}}
\newcommand{\cN}{\mathcal{N}}
\newcommand{\cO}{\mathcal{O}}
\newcommand{\cR}{\mathcal{R}}
\newcommand{\cV}{\mathcal{V}}
\newcommand{\cX}{\mathcal{X}}

\newcommand{\cY}{\mathcal{Y}}

\newcommand{\bD}{\mathbb{D}}
\newcommand{\bE}{\mathbb{E}}
\newcommand{\bI}{\mathds{1}}

\newcommand{\bP}{\mathbb{P}}
\newcommand{\bR}{\mathbb{R}}
\newcommand{\bU}{\mathbb{U}}

\newcommand{\argmin}{\mathrm{argmin}}

\newcommand{\dx}{\mathrm{d}x}
\newcommand{\dy}{\mathrm{d}y}
\newcommand{\dz}{\mathrm{d}z}
\newcommand{\dt}{\mathrm{d}t}

\newcommand{\nhyp}{n}

\newcommand{\bn}{[\![1,\nhyp]\!]}

\title{Representation Gap: Explaining the Unreasonable Effectiveness of Neural Networks from a Geometric Perspective}

\author{%
David Perera\\
Universidade Federal de Minas Gerais\\
Belo Horizonte, Brazil
\\
\And
Victor Moura\\
Universidade Federal de Minas Gerais\\
Belo Horizonte, Brazil\\
\And
Lais Isabelle Alves dos Santos\\
Universidade Federal de Minas Gerais\\
Belo Horizonte, Brazil\\
\And
Michel F. C. Haddad\\
Queen Mary University of London\\
London, United Kingdom\\
\And
Flavio Figueiredo\\
Universidade Federal de Minas Gerais\\
Belo Horizonte, Brazil\\
}

\begin{document}

\maketitle

\begin{abstract}
  Characterizing precisely the asymptotic generalization error of neural networks using parameters that can be estimated efficiently is a crucial problem in machine learning, which relies heavily on heuristics and practitioners' intuition to make key design choices. In order to mitigate this issue, we introduce the Representation Gap, a metric closely related to the generalization error, but admitting better-behaved asymptotic dynamics. Focusing on equivariant diffusion models and leveraging results from optimal quantization and point-process theory, we derive a precise asymptotic equivalent of the Representation Gap and show that it is governed by a single parameter, the \textit{intrinsic dimension} of the task, which is easy to interpret, efficient to estimate, and can be linked to the equivariances of common neural network architectures. We show that this asymptotic dynamic also extends to a broader range of tasks and training algorithms. Finally, we demonstrate empirically that our asymptotic law and intrinsic dimension estimation are accurate on a wide range of synthetic datasets, where these quantities are known, as well as on more realistic datasets, where we obtain results consistent with the related literature.\footnote{Code for reproducing our experiments is available at \url{https://github.com/daperera/representation_gap}.}
\end{abstract}


\section{Introduction}

Neural networks combine strong memorization capabilities with architectural and optimization biases that shape their behavior outside of the training dataset \citep{hornik_approximation_1991,kaplan_scaling_2020,kubo_implicit_2019,zhang_understanding_2021}. In practice, these inductive biases are often aligned with the geometry and symmetries of real-world tasks \citep{fefferman_testing_2016,chiang_loss_2022,teney_neural_2024}. 
As a result, neural networks effectively augment the training data and can generalize well beyond simple memorization \citep{zhang_understanding_2021,allen-zhu_convergence_nodate,belkin_fit_2021,simon2026there}. Recent work on diffusion models has even shown that the outputs of trained equivariant architectures can be predicted accurately from the training data and the symmetries of the model alone \citep{kamb_analytic_2025,finn_origins_2025}. These observations suggest that the generalization capabilities of neural networks are largely determined by the geometry of the data manifold and the symmetries of the model. Our goal in this paper is to characterize neural network generalization from this geometric perspective, using measurable properties of the data and the model.

Equivariant architectures are typically analyzed by controlling the generalization error with PAC and generalization bounds \citep{chen_group-theoretic_nodate,elesedy_provably_2021,tahmasebi_exact_nodate}. However, these bounds are not always tight, and often depend on quantities that are difficult to estimate in practice (e.g. intrinsic dimension \citep{ansuini_intrinsic_nodate,gong_intrinsic_2019}). Moreover, the generalization error is well defined for prediction tasks, but harder to extend to other tasks such as generative modeling \citep{theis2015note}. 

Motivated by these limitations, we introduce the \textit{representation gap} $\mathcal{R}(\Omega,\Omega_f)$, which measures the discrepancy between the data manifold $\Omega$ and the prediction space $\Omega_f$ of a trained model $f$. The representation gap extends the generalization error to prediction tasks and generative modeling within a unified framework. We demonstrate that it admits a surprisingly simple asymptotic scaling in $n^{-2/d}$, where $n$ is the size of the training dataset $\mathbb{D}$ and $d$ is an \textit{intrinsic dimension} parameter that depends only on the geometry of $\Omega$ and the symmetries of $f$. As a corollary, we demonstrate how model equivariance reduces this intrinsic dimension $d$, thereby provably improving generalization.

Theoretical analyses of generalization typically assume that training and test data are \textit{i.i.d.} \citep{shalev2014understanding}, and we follow this standard framework. However, since real-world datasets are often collected with the goal of covering the diversity of the task \citep{deng_imagenet_nodate,lin2014fleet,torralba2011unbiased}, we also formulate our results for optimally diverse datasets \citep{zador1982asymptotic}. Interestingly, we show that \textit{i.i.d.} datasets exhibit the same asymptotic behavior as optimally diverse datasets, up to a rescaling of the effective sample size $n_{\mathrm{eff}}$. Overall, we make the following contributions.

\textbf{We introduce the representation gap}, a geometric quantity that extends the generalization error to prediction tasks and generative modeling within a unified framework.

\textbf{We derive precise asymptotic equivalents of the representation gap} for equivariant diffusion models. We extend this result to the setting of supervised prediction, and establish bounds relating representation gap and generalization error. Our results hold both for \textit{i.i.d.} datasets and optimally diverse datasets.

\textbf{We show that asymptotic representation gap is governed by the intrinsic dimension of the task}, a single parameter determined by the geometry of the data manifold and the symmetries of the model. We further show that this intrinsic dimension can be estimated efficiently.

\textbf{We validate our theoretical predictions} on controlled synthetic environments with known intrinsic dimension, as well as on more realistic datasets.

\section{Related work}

\textbf{Geometric perspective on generalization}. Building on the manifold hypothesis \citep{bengio_representation_2013}, several works have studied neural networks as manifold learners \citep{loaiza-ganem_deep_2024,schuster_manifold_2021}. Focusing on ReLU networks, the authors of \citet{yao_theoretical_2024} derive generalization bounds based on geometric properties of the data manifold, such as its dimension or Betti numbers. In contrast, we derive precise asymptotic equivalents and relate them to model equivariances. We further compare our intrinsic dimension estimator with prior manifold dimension estimators \citep{pope_intrinsic_2021,gong_intrinsic_2019,ansuini_intrinsic_nodate} and obtain consistent estimates across several datasets (see Section \ref{sec:dimension_estimation}).

\textbf{Generalization of equivariant neural networks}. Empirical studies have shown that equivariance improves generalization and sample efficiency \citep{cohen_group_nodate,bulusu_equivariance_2022}. 
A large body of work studies PAC and generalization bounds \citep{sannai_improved_nodate,chen_group-theoretic_nodate,elesedy_provably_2021}. Closest to our work, \citet{tahmasebi_exact_nodate} show that the generalization error of Kernel Ridge Regression is bounded by $n^{-s/(s+d/2)}$, where $d$ is  the dimension of the quotient manifold induced by the model symmetries. In contrast, we derive asymptotic equivalents for the representation gap and recover related bounds on the generalization error as a corollary (see Section \ref{sec:generalization_error}). Finally, \citet{kamb_analytic_2025} derive analytic expressions for the predictions of trained diffusion models, which underlies part of our analysis.

\textbf{Scaling laws}. Our work is related to neural scaling laws \citep{kaplan2020scaling} and recent studies on diffusion model scaling \citep{mei2024bigger,li2024scalability,liang2024scaling}. While prior work mainly studies empirical scaling with respect to compute, we focus on the geometric scaling induced by dataset size and model equivariance.

\textbf{Optimal quantization and point processes.} Our analysis relies on point process theory in the \textit{i.i.d.} setting \citep{biau2015lectures,penrose2013limit} and optimal quantization theory in the optimally diverse setting \citep{gruber_optimal_nodate}. However, the representation gap combines geometric and statistical aspects, requiring substantial adaptation of existing results.

\section{An illustrative example}\label{sec:illustrative_example}

\begin{figure}
    \centering
    \begin{tabular}{ccc}
        \includegraphics[width=0.35\linewidth]{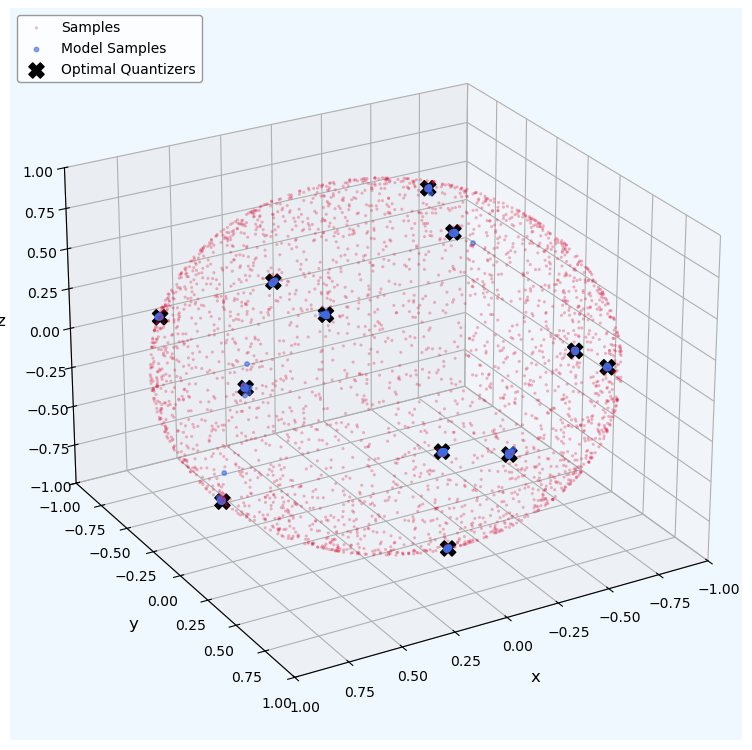} &
        \includegraphics[width=0.35\linewidth]{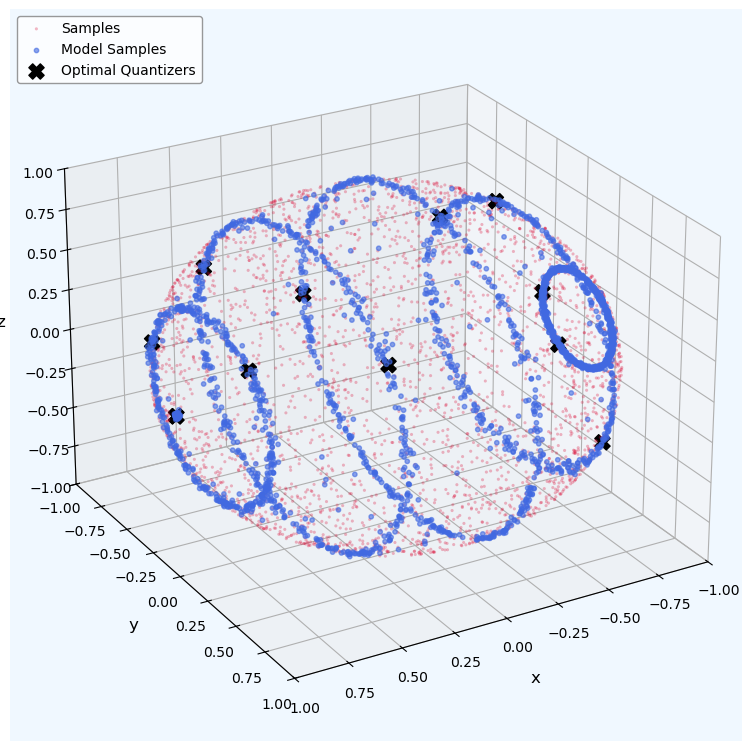}  
        \\
        \small (a) Non-equivariant model &
        \small (b) Equivariant model 
        \\
    \end{tabular}
    \caption{Illustration of the virtual augmentation of a dataset by an equivariant diffusion model.  Plot (a) shows samples from a trained diffusion model, and plot (b) shows samples from a trained equivariant diffusion model (with rotational invariance along the $x$-axis). In both plots, the shape $\Omega$ is indicated by a dense cloud of red dots, the coarse dataset $\bD$ by crosses, and the approximated shape $\Omega_f$ by a dense cloud of blue dots sampled from the trained diffusion model $f$.}
    \label{fig:illustrative_example}
\end{figure}

\begin{wrapfigure}{l}{0.45\textwidth}
\vspace{-15pt}
  \begin{center}
    \includegraphics[width=0.42\textwidth]{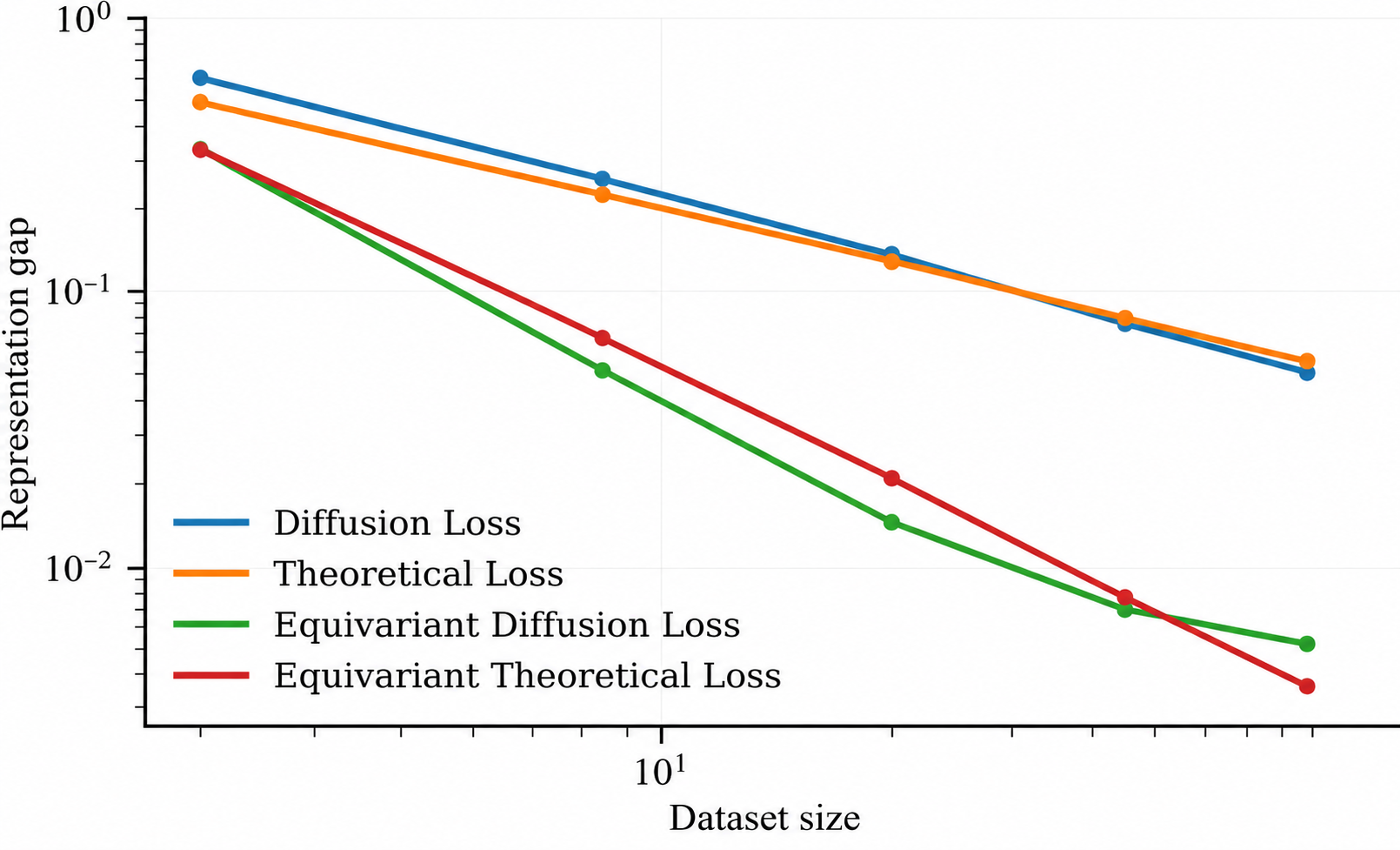}
  \end{center}
  \caption{Log plot of the asymptotic evolution of the representation gap of a rotation-equivariant model and a non-equivariant model for a 2D sphere surface. The $x$-axis corresponds to the dataset size $n$, and the $y$-axis corresponds to the representation gap. We observe a linear evolution, with slope $-1$ for the non-equivariant model and $-2$ for the equivariant model. The theoretical curves are shown using an empirical estimate of the multiplicative constant J in Eq. \ref{eq:simplified_asymptotic_gap} }
  \label{fig:sphere_representation_gap}
 \vspace{-10pt}
\end{wrapfigure}

Let us first introduce the main concepts of this paper through a concrete example. We consider the task of generative modeling of 3D shapes \citep{yang_pointflow_2019}. The goal is to learn to sample points $y$ from a surface $\Omega\subset\bR^3$ that is described by a coarse $n$-point cloud $\bD\in\Omega^n$. Diffusion models have recently achieved strong empirical performance on this task \citep{li_advances_2024}. We denote by $\Omega_f$ the set of points that a trained diffusion model $f$ can generate -- in other words, the limit points of the denoising process.

This setting is illustrated by Figure \ref{fig:illustrative_example}. The surface $\Omega$ is represented by a dense cloud of red dots, the coarse dataset $\bD$ by crosses, and the prediction space $\Omega_f$ by a dense cloud of blue dots sampled from a trained diffusion model $f$. In this example, the surface $\Omega$ exhibits a rotational symmetry, which reduces the degrees of freedom of the task. A natural way to leverage this symmetry is to use a rotation-equivariant diffusion model $f$ \citep{hoogeboom_equivariant_nodate}. Figure \ref{fig:illustrative_example}(a) shows the output of a non-equivariant model, while Figure \ref{fig:illustrative_example}(b) shows the output of an equivariant model.

We make the following two observations. First, the distribution learned by the non-equivariant neural network converges toward the empirical distribution $\frac{1}{|\bD|}\sum_{y\in\bD}\delta_y$, so that the prediction space $\Omega_f$ coincides with the dataset $\bD$. In other words, $\Omega_f=\bD$. In contrast, the equivariant model virtually augments the dataset $\bD$ by the rotation group $G$ under which it is equivariant, so that $\Omega_f=G(\bD)=\{g(z)|z\in\bD,g\in G\}$.

It is clear from Figure \ref{fig:illustrative_example} that equivariance drastically improves the resolution of the prediction space $\Omega_f$. In order to quantify this improvement, we introduce the representation gap, a measure of how well the prediction space $\Omega_f$ approximates the data manifold $\Omega$ under a metric $\ell$ defined on the ambient space. In this work, $\ell$ denotes the squared Riemannian distance unless stated otherwise.

\begin{definition}[Representation gap]\label{def:representation_gap}
    Let $\Omega$ denote the data manifold and $\Omega_f$ denote the model's prediction space. We define the representation gap as follows:
    \begin{equation}\label{eq:representation_gap}
        \cR(\Omega,\Omega_f)=\int_\Omega\inf_{z\in\Omega_f}\ell(y,z)p(y)\;\dy\;.
    \end{equation}
\end{definition}

Concretely, Eq. \ref{eq:representation_gap} projects each sample $y\in\Omega$ to the closest prediction point $z\in\Omega_f$ generated by the model, and averages this error across the data manifold. It is worth noting that the representation gap is a special case of the Wasserstein distance \citep{peyre2019computational} (see Section \ref{appsec:link_wasserstein} in Appendix), which is commonly used to compare sets, as well as a natural generalization of the quantization error, which we recover when the set $\Omega_f$ is discrete \citep{graf2007foundations}. 

Intuitively, a non-equivariant model $f$ requires information about all the $d_\Omega=2$ dimensions of the shape $\Omega$ in order to approximate it from the dataset $\bD$ (as illustrated on the left of Figure \ref{fig:illustrative_example}). On the other hand, the equivariant model only needs information along the rotational axis, with dimension $d_\Omega-1=1$. More generally, for an arbitrary manifold $\Omega$ and symmetry group $G$, the equivariant model only needs information about the quotient space $\Omega/G$, with dimension $d_{\Omega/G}$. The remaining dimensions are implicitly recovered by the virtual augmentation of the dataset, since $\Omega_f=G(\bD)$. This leads to the following asymptotic characterization of the representation gap, illustrated in Figure \ref{fig:sphere_representation_gap}.

\begin{theorem}[Asymptotic representation gap -- informal statement]\label{prop:efficiency_informal}
    See Theorems \ref{prop:efficiency_nonconditional} and \ref{prop:efficiency_conditional}. 
    The representation gap $\cR_n$ of a model trained on a dataset $\bD$ of size $n$ scales as
    \begin{equation}\label{eq:simplified_asymptotic_gap}
        \cR_n \;\underset{n\rightarrow+\infty}{\scalebox{1.5}{$\sim$}}\; \frac{J}{n^{2/d}}\;, 
    \end{equation}
    where $d$ denotes the intrinsic dimension of the task: $d=d_\Omega$ for a non-equivariant model or $d=d_{\Omega/G}$ for an equivariant model. The constant $J$ admits an analytic expression depending only on the geometry of the manifold $\Omega$, the symmetry group $G$ and the quotient metric on $\Omega/G$.
\end{theorem}

The asymptotic evolution of the representation gap $\cR_n$ is governed by the single parameter $d$, which we name intrinsic dimension. In particular, this result characterizes precisely the advantage of the equivariant model over the non-equivariant one, since equivariance improves the asymptotic scaling whenever $d_{\Omega/G}<d_\Omega$. The next Section formalizes these observations for trained equivariant diffusion models (Theorems \ref{prop:virtual_augmentation} and \ref{prop:efficiency_nonconditional}) and extends our analysis to the setting of supervised prediction (Theorem \ref{prop:efficiency_conditional}).

\section{Theoretical results}\label{sec:theory}

\subsection{Preliminaries}\label{sec:preliminaries}

We first consider the task of non-conditional diffusion modeling and formalize the claims of Section \ref{sec:illustrative_example}. Let $\cY=\bR^{d_\cY}$ denote the target space, of dimension $d_\cY$. Under the manifold hypothesis \citep{bengio_representation_2013}, observations are assumed to lie on a low-dimensional Riemannian manifold $\Omega\subset\cY$ of dimension $d_\Omega$, whose geometry captures the symmetries of the task. We further suppose access to a dataset $\bD\subset\Omega$ composed of $n$ observations drawn from a distribution $p$ supported on $\Omega$. We consider neural networks $f_\theta$ in a parametric family $\cF_\Theta\subset\cF(\cY\times\bR,\cY)$, which we simply denote by $f$ when there is no ambiguity.

For simplicity, we follow \citet{kamb_analytic_2025} and \citet{finn_origins_2025}, and focus on Denoising Diffusion Implicit Models (DDIM) diffusion models \citep{song_denoising_2022}. 
DDIM models are trained to reverse a stochastic forward diffusion process that incrementally adds Gaussian noise to the data distribution while shrinking data points toward the origin. Noise addition is governed by a noise schedule $\alpha_t$, with $t\in[0,T]$. At time $t$, the noised distribution can be written $\pi_t(y)=\frac{1}{|\bD|}\sum_{z\in\bD}\cN(y|\sqrt{\alpha_t}z, (1-\alpha_t)I)$, thus interpolating between the empirical data distribution $\pi_0=\frac{1}{|\bD|}\sum_{z\in\bD}\delta_z$ and the isotropic Gaussian distribution $\pi_T=\cN(0,I)$. In this context, DDIM models are trained to approximate the score function $s_t=\nabla \log\pi_t$ using the loss
\begin{equation}\label{eq:diffusion_training}
    \cL(\theta)=\bE_{t,y_0,\eta}\|f_\theta(\sqrt{\alpha_t}y_0+\sqrt{1-\alpha_t}\eta,t)-\eta\|_2^2\;,
\end{equation}
where $t\sim \bU[0,T]$, $y_0\sim\pi_0$ and $\eta\sim\cN(0,I)$.
At sampling time, an initial point $y_T\sim\cN(0,I)$ is sampled and then updated using the deterministic flow 
\begin{equation}\label{eq:diffusion_sampling}
    \dot{y}_t=-\gamma_t(y_t+s_t(y_t))\;,
\end{equation}
where $t$ goes backward from $T$ to $0$. The outputs correspond to the endpoints reachable by this reverse flow, i.e.:
\begin{equation}
   \Omega_f\triangleq\{y_0 \;|\; y_T\sim\cN(0,I)\;,\; y_t \text{ solves Eq. } \ref{eq:diffusion_sampling} \}\;. 
\end{equation}

\subsection{Virtual augmentation of a dataset by an equivariant model}\label{sec:virtual_augmentation}

A diffusion model minimizing the training objective $\cL$ exactly —- and therefore recovering the true score function $s_t$ -— generates samples following the empirical distribution $\pi_0=\frac{1}{|\bD|}\sum_{z\in\bD}\delta_z$ \citep{song_generative_nodate}. In this case, the prediction space $\Omega_f$ learned by the model $f$ is the training data itself. Therefore, $\Omega_f$ provides a discrete approximation of the data manifold $\Omega$ given by $\Omega_f=\bD$. 

In practice, however, the neural network family $\cF_\Theta$ has a limited expressivity, which prevents the perfect estimation of the true score $s_t$. Instead, neural network architectures are often designed to enforce the symmetries of the task. Remarkably, it is possible to show following \citet{kamb_analytic_2025} that these architectural constraints induce a virtual augmentation of the training dataset $\bD$ by the symmetry group $G$ induced by the architecture, so that we have in effect $\Omega_f=G(\bD)$.

\begin{theorem}[Virtual augmentation of a dataset by an equivariant model]\label{prop:virtual_augmentation}
    See Proposition \ref{app_prop:virtual_augmentation} in Appendix. Let $f$ denote a diffusion model equivariant under a symmetry group $G$ and minimizing the training objective in Eq. \ref{eq:diffusion_training} on a dataset $\bD$. Then under mild assumptions on $G$, $\Omega$ and $\bD$, the set of points that can be predicted by $f$ is $\Omega_f=G(\bD)$.
\end{theorem}

\begin{proof}
    The proof of Theorem \ref{prop:virtual_augmentation} relies on the following observation:  the score function $s_t$ at a point $y\in\cY$ can be written as an integral over the orbits $G(\bD)$ of the dataset $\bD$: 
    \begin{equation*}
        s_t(y)=-\frac{1}{1-\alpha_t}\int_{G(\bD)}(y-\sqrt{\alpha_t}z)W_t(z)\dz\;,
    \end{equation*}
    where each point $z\in G(\bD)$ is weighted by the distribution 
    \begin{equation*}
        W_t(z)=\frac{\cN\left(y|\sqrt{\alpha_t}z,(1-\alpha_t)I\right)}{\int_{G(\bD)}\cN\left(y|\sqrt{\alpha_t}z',(1-\alpha_t)I\right)\dz'}\;.
    \end{equation*} 
    We can see that $W_t(y)$ acts as a softmax that peaks at the minimizer $y^*=\argmin_{z\in G(\bD)}\ell(y,z)$ for small $t$. More precisely, we can use a Laplace approximation to show that $W_t(y)$ concentrates the probability mass around $y^*$ when $t\rightarrow0$. 

    Under the hypothesis that $f$ minimizes the training objective in Eq. \ref{eq:diffusion_training}, we can therefore write
    \begin{equation*}
        f(y_t,t)=-\frac{1}{1-\alpha_t}\int_{G(\bD)}(y_t-\sqrt{\alpha_t}z)W_t(z)\dz = \frac{1}{1-\alpha_t}(y_t-y_t^*)+o\left(\frac{1}{1-\alpha_t}\right)\;, 
    \end{equation*}
    which in turn implies $y_t-y^*_t\approx(1-\alpha_t) f(y_t,t)\rightarrow0$, and therefore $\lim_{t\rightarrow0}y_t=\lim_{t\rightarrow0}y_t^*\in G(\bD)$ (by properties of $G$). This proves $\Omega_f\subset G(\bD)$. The reverse inclusion is detailed in Appendix.
\end{proof}

\subsection{Representation gap for non-conditional diffusion}\label{sec:non_conditional}

Using Theorem \ref{prop:virtual_augmentation}, we can now characterize the asymptotic representation gap in the large sample regime. Crucially, since equivariant architectures virtually augment the dataset by $\Omega_f=G(\bD)$, the representation gap no longer depends on the ambient manifold $\Omega$, but only on the geometry of the quotient manifold $\Omega/G$. The representation gap therefore reduces to a quantization problem on the quotient manifold, allowing the use of asymptotic results from optimal quantization and point processes theory. 

The representation gap $\cR(\Omega,\Omega_f)$ depends on how the dataset $\bD$ is sampled from $\Omega$. It is typical to assume that $\bD$ is a dataset of size $n$ sampled \textit{i.i.d.} from the data distribution $p$, and we denote 
\begin{equation}
    \cR_n\triangleq\cR(\Omega,\Omega_f(\bD))
\end{equation}
the corresponding \textit{random representation gap}. 
In practice, however, datasets are often collected to cover the diversity of the task, modulo its known invariants \citep{torralba2011unbiased}. Motivated by this observation, we also consider the setting where $\bD$ is optimally diverse, i.e. minimizes the representation gap, and denote the \textit{optimal representation gap} by
\begin{equation}
    \cR_n^*\triangleq\inf_{\bD\subset\Omega,\; |\bD|=n}\cR(\Omega,\Omega_f(\bD))
\end{equation}

The following result is a formalization of Theorem \ref{prop:efficiency_informal}.

\begin{theorem}[Representation gap for non-conditional diffusion]\label{prop:efficiency_nonconditional}
    See Propositions \ref{app_prop:optimal_efficiency_general},\ref{app_prop:optimal_efficiency_manifold}, \ref{app_prop:random_efficiency_manifold} and \ref{app_prop:efficiency_equivariant} in Appendix. 
    Let $f$ denote an equivariant model satisfying $\Omega_f=G(\bD)$, where  $\bD$ is a dataset of size $n$. Suppose further that the orbits $G(y)$ have constant volume for each point $y\in\Omega$.
    Then under mild regularity assumptions on $\Omega$ and $G$, the representation gap satisfies
    \begin{equation}\label{eq:efficiency_nonconditional}
    \begin{aligned}
    \text{(i.i.d.)}\quad 
    \cR_n &\sim_{\mathbb{P}} \frac{J_d}{n^{2/d}}
    \qquad\qquad
    \text{(optimal)}\quad 
    \cR_n^{\star} \sim \frac{J_d^*}{n^{2/d}}
    \end{aligned}
    \end{equation}
    where $d=d_{\Omega/G}$ denotes the dimension of $\Omega/G$, the quotient space of $\Omega$ by the symmetry group $G$, and the constants $J_d$ and $J_d^*$  depend only on the quotient geometry and data distribution on $\Omega/G$. In particular, equivariance improves the asymptotic scaling whenever $d_{\Omega/G}<d_\Omega$.
\end{theorem}

\begin{proof}
We have $\Omega_f=G(\bD)$. Using the orbit decomposition of $\Omega$ and the isometric action of $G$ (see for instance \citet{gallot_riemannian_1990}), the representation gap reduces to
\begin{equation*}
\cR(\Omega,\Omega_f) =|G| \int_{\Omega/G} \min_{z\in\bD} \ell_{\Omega/G}(y,z)\,p(y)\,\dy,
\end{equation*}
so that the problem reduces to quantization on the quotient manifold $\Omega/G$ . The asymptotic optimal representation gap then follows from Zador's theorem (see Theorem 2 in \citet{gruber_optimal_nodate}). The \textit{i.i.d.} setting is treated by Proposition \ref{app_prop:random_efficiency_manifold} in Appendix.
\end{proof}

Theorem \ref{prop:efficiency_nonconditional} provides a precise asymptotic equivalent of the representation gap, which is remarkable since most existing analyses of neural network generalization focus on bounds \citep{zhang_understanding_2021}. In particular, the convergence in probability on the left of Eq. \ref{eq:efficiency_nonconditional} is a strong result, which implies that the representation gap of an \textit{i.i.d.} dataset is asymptotic close to its equivalent $J_d\,n^{-2/d}$ with arbitrarily high probability. The constants $J_d$ and $J_d^\star$ admit analytic expressions (see Prop. \ref{app_prop:optimal_efficiency_manifold} and \ref{app_prop:random_efficiency_manifold} in Appendix).

As a direct corollary, we obtain $\cR_n \sim_{\bP} \cR_{n_{\mathrm{eff}}}^\star$,    
with effective sample size $n_{\mathrm{eff}}=\left(J_d^\star/J_d\right)^{d/2}$. Thus, random datasets exhibit the same asymptotic behavior as optimally diverse datasets, up to a rescaling of the effective sample size $n_{\mathrm{eff}}$.

As a corollary of Theorem \ref{prop:virtual_augmentation}, Theorem \ref{prop:efficiency_nonconditional} applies to equivariant diffusion models minimizing the training objective in Eq. \ref{eq:diffusion_training}, under the regularity assumptions of Theorem \ref{prop:virtual_augmentation}. However, the result applies more generally to any generative model satisfying $\Omega_f=G(\bD)$, and is therefore not restricted to diffusion models. It also extends naturally to conditional generative models when the conditioning variable takes finitely many values (see Proposition \ref{app_prop:efficiency_discrete} in Appendix).

\subsection{Representation gap for supervised prediction}\label{sec:conditional}

We now turn to the more general setting of supervised prediction. Each input $x\in\Omega_\cX$ is associated with a unique target $y(x)\in\Omega_\cY$, so that the observation manifold $\Omega\subset\Omega_\cX\times\Omega_\cY$ can be identified with the graph of the function $y:\Omega_\cX\rightarrow\Omega_\cY$. In particular, the intrinsic dimension of $\Omega$ coincides with that of the input manifold, \textit{i.e.}, $d_\Omega=d_{\Omega_\cX}$, independently of the dimension of $\cY$. We further assume that the model $f$ generates a unique prediction $f(x)$ for each input $x\in\Omega_\cX$, so that the prediction manifold $\Omega_f$ can similarly be identified with the graph of $f$. In this context, the conditional representation gap is defined by
\begin{equation}\label{eq:representation_gap_conditional}
    \cR(\Omega,\Omega_f)
    =
    \int_{\Omega_\cX}
    \min_{z'\in\Omega_f}
    \ell(z,z')
    \,p(x)\,dx,
\end{equation}
where $z=(x,y(x))\in\Omega$ and $z'=(x',f(x'))\in\Omega_f$. We further denote by $\ell_\cX$ and $\ell_\cY$ the metrics induced by $\ell$ on $\cX$ and $\cY$ respectively.

\begin{theorem}[Conditional representation gap of an equivariant model]\label{prop:efficiency_conditional}
    See Proposition \ref{app_prop:conditional_efficiency_equivariant}. 
    Let $f$ denote an equivariant $L$-Lipschitz model satisfying $\Omega_f=G(\bD)$, where $\bD$ is a training dataset of size $n$ and $L>0$. Suppose further that the orbits $G(x)$ have constant volume for each point $x\in\Omega_\cX$, and that the metric $\ell$ is additively separable on $\cX$ and $\cY$. 
    Then under mild regularity assumptions, the representation gap satisfies
    \begin{equation}\label{eq:efficiency_conditional}
    \begin{aligned}
    \text{(i.i.d.)}\quad 
    \cR_n &= O_{\mathbb{P}} \left(\frac{1}{n^{2/d}}\right)
    \qquad\qquad
    \text{(optimal)}\quad 
    \cR_n^{\star} &= O\left(\frac{1}{n^{2/d}}\right)
    \end{aligned}\;,
    \end{equation}
    where $\Omega_\cX/G$ denotes the quotient space of $\Omega_\cX$ by the symmetry group $G$, and $d=d_{\Omega_\cX/G}$ denotes the dimension of $\Omega_\cX/G$.
\end{theorem}

\begin{proof}
Using the equivariance of $f$, we proceed as in the proof of Theorem \ref{prop:efficiency_nonconditional} and reduce the representation gap to a quantization problem on the quotient manifold $\Omega_\cX/G$. Then, let $z=(x,y(x))\in\Omega$ denote a data sample with input $x$, let $\hat x=\argmin_{x'\in\bD_\cX}\ell_\cX(x,x')$ denote the nearest training input to $x$, and let $\hat{z}=(\hat{x},y(\hat{x}))$ denote the corresponding training sample. Since $f$ interpolates the training dataset, we have $f(\hat x)=y(\hat x)$. Using the additive separability of $\ell$ and the Lipschitzness of $f$,
\begin{equation*}
  \ell \big(z,\hat{z}\big)
\leq
\ell_\cX(x,\hat x)
+
\ell_\cY(f(x),f(\hat x))
\leq
(1+L)\ell_\cX(x,\hat x)\;.  
\end{equation*}

Thus, $\cR(\Omega,\Omega_f)
\leq
(1+L)
\int_{\Omega_\cX}
\min_{x'\in\bD_\cX}
\ell_\cX(x,x')
\,p(x)\,dx,$
and the result follows from Theorem \ref{prop:efficiency_nonconditional}.
\end{proof}

\subsection{Comparison with generalization error}\label{sec:generalization_error}

A natural question is to relate the representation gap $R(\Omega,\Omega_f)$ to the generalization error \citep{shalev2014understanding},
commonly used to characterize generalization. We focus on the setting of prediction tasks, for which there is a widely accepted definition of the generalization error, 
$\mathcal{E}=\int_{\Omega}\ell_\cY(y(x),f(x))\,p(x)\,\dx$.

\begin{theorem}[Comparison with generalization error]\label{prop:generalization_error}
    See Proposition \ref{app_prop:generalization_error_link} in Appendix. If the model $f$ is $L$-Lipschitz and the metric $\ell$ is additively separable on $\cX$ and $\cY$, we have under mild regularity assumptions that
    \begin{equation}\label{eq:generalization_error}
        \frac{1}{1+L}\mathcal{E}\leq \cR(\Omega,\Omega_f)\leq\mathcal{E}\;.
    \end{equation}
\end{theorem}

Combining Theorems \ref{prop:efficiency_conditional} and \ref{prop:generalization_error}, we obtain $\mathcal{E} = O\left(n^{-2/d_{\Omega}}\right)$ as $n\rightarrow+\infty$, a result closely related to \citet{tahmasebi_exact_nodate}. Moreover, $\cR(\Omega,\Omega_f)=0$ implies $f(x)=y_x$ almost everywhere, and therefore $\mathcal{E}=0$. Generalization error and representation gap are therefore closely related. 

\section{Experimental results}\label{sec:experiments}

We now validate experimentally the theoretical results of Section \ref{sec:theory}.

\begin{figure}[t]
    \centering
    \setlength{\tabcolsep}{3pt}
    \begin{tabular}{cc|cc}
        \multicolumn{2}{c}{\small Hypercube} & \multicolumn{2}{c}{\small Wave} \\
        \includegraphics[width=0.23\linewidth]{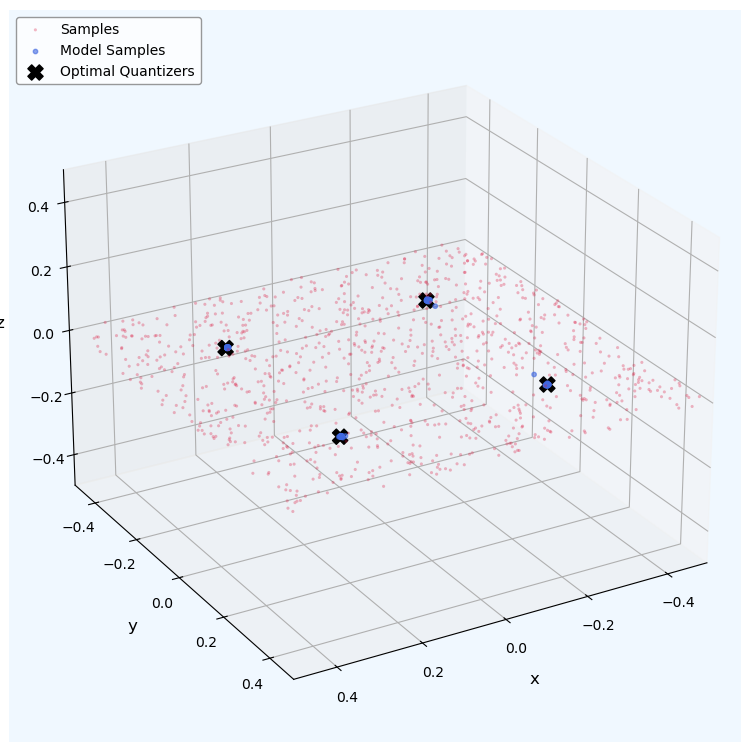} &
        \includegraphics[width=0.23\linewidth]{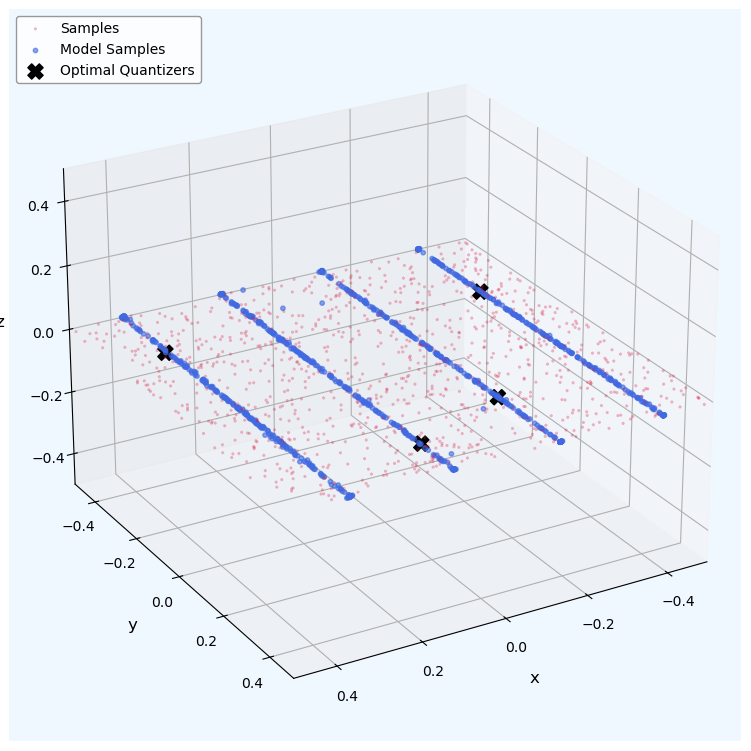} &
        \includegraphics[width=0.23\linewidth]{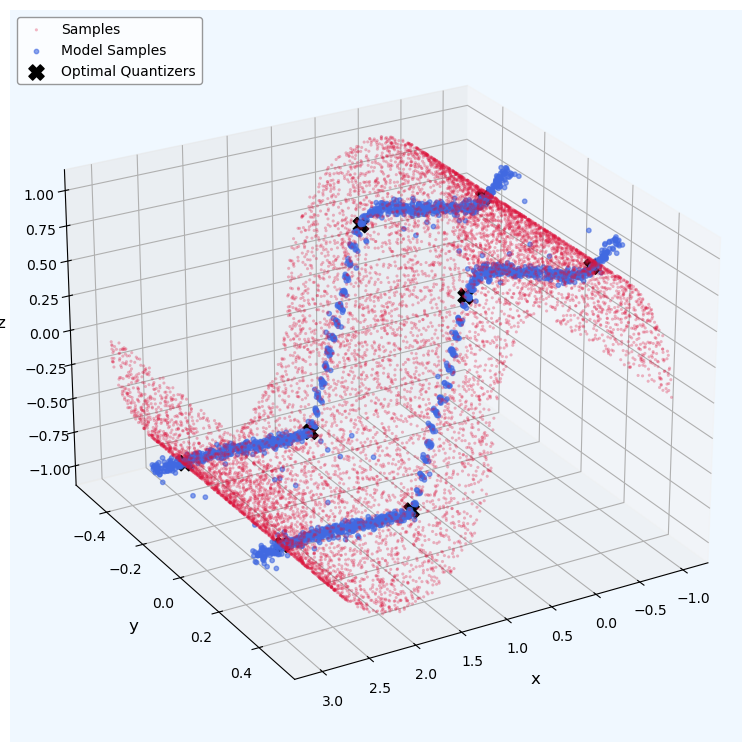} &
        \includegraphics[width=0.23\linewidth]{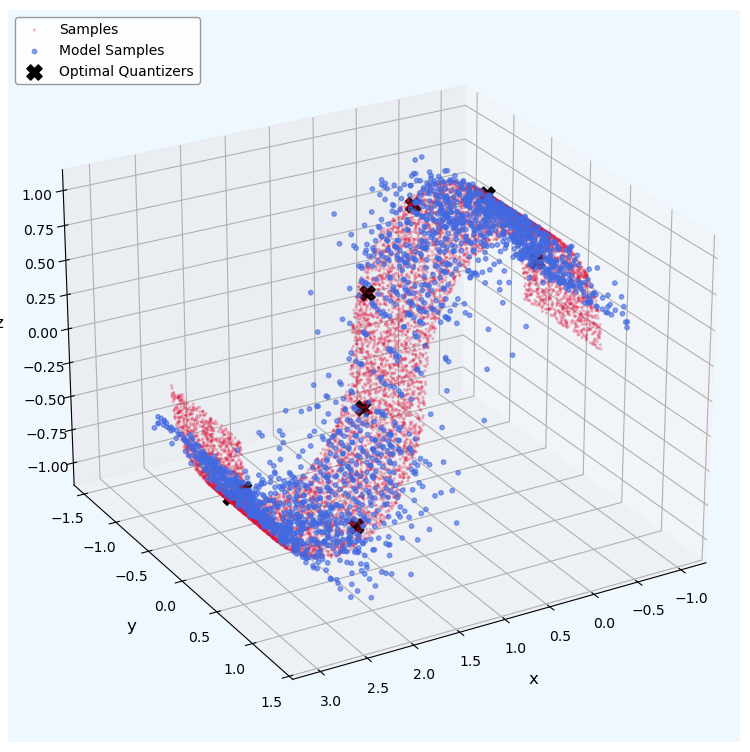} \\
        \small (a) Normal & \small (b) Equivariant &
        \small (c) Normal & \small (d) Equivariant \\
    \end{tabular}
    \caption{Virtual augmentation of a dataset in the non-conditional setting (Hypercube dataset) and the conditional setting (Wave dataset). Plots (a) and (c) show samples from a trained non-equivariant diffusion model, while plots (b) and (d) show samples from a trained equivariant diffusion model (with translation invariance along the $y$-axis). For the Wave dataset, the conditioning corresponds to the $x$-axis. 
    }
    \label{fig:overfitting_example}
\end{figure}

\textbf{Datasets.} We conduct experiments on several synthetic datasets for non-conditional generative modeling and supervised prediction.  
They are illustrated in Figure \ref{fig:illustrative_example} and Figure \ref{fig:overfitting_example}. 

\noindent \textit{Hypercube} corresponds to a $d_\Omega$-dimensional hypercube $\Omega=\left[-\frac{c}{2},\frac{c}{2}\right]^{d_\Omega}$ of side $c$ embedded into a $d_\cY$ ambient space. This dataset is translation-invariant along each dimension.

\noindent \textit{Hypersphere} corresponds to a 2-dimensional hypersphere $\partial B(0,r)$ of radius $r$ embedded into a 3-dimensional ambient space. This dataset features rotational invariances along axes $x$, $y$ and $z$. 

\noindent \textit{Wave} is a 2-dimensional wave surface embedded into a 3-dimensional ambient space, obtained by translating a $(x,z)$-curve composed of half-circles along the $y$-axis.
This dataset corresponds to a conditional prediction task, where $x$ is the input and $(y,z)$ is the target. It features translation invariance over $y$.

\begin{wrapfigure}{r}{0.36\textwidth}
\vspace{-1em}
\centering
\scriptsize
\setlength{\tabcolsep}{2pt}
\begin{tabular}{c}
    \includegraphics[width=0.95\linewidth]{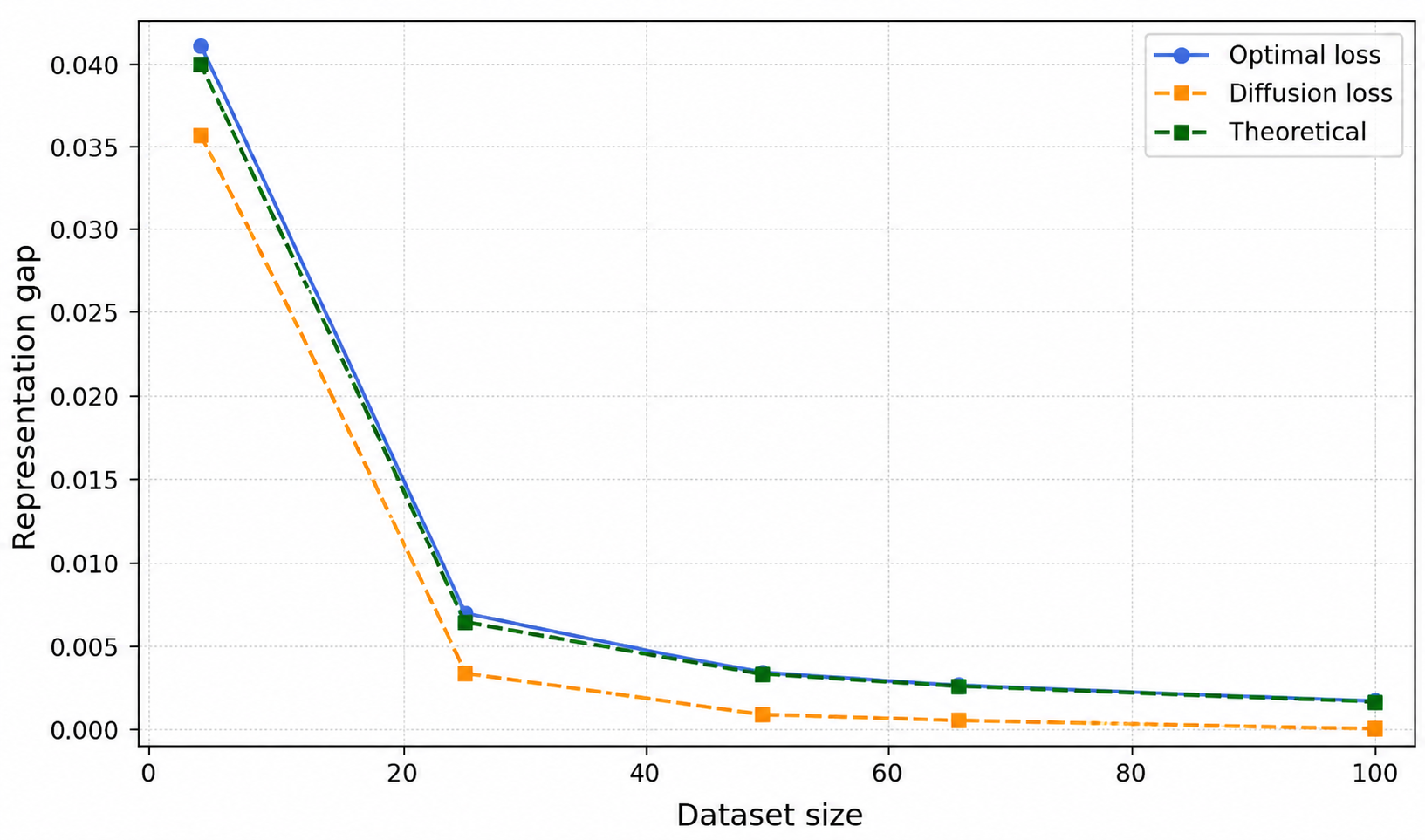} \\
    (a) Hypercube \\
    \includegraphics[width=0.95\linewidth]{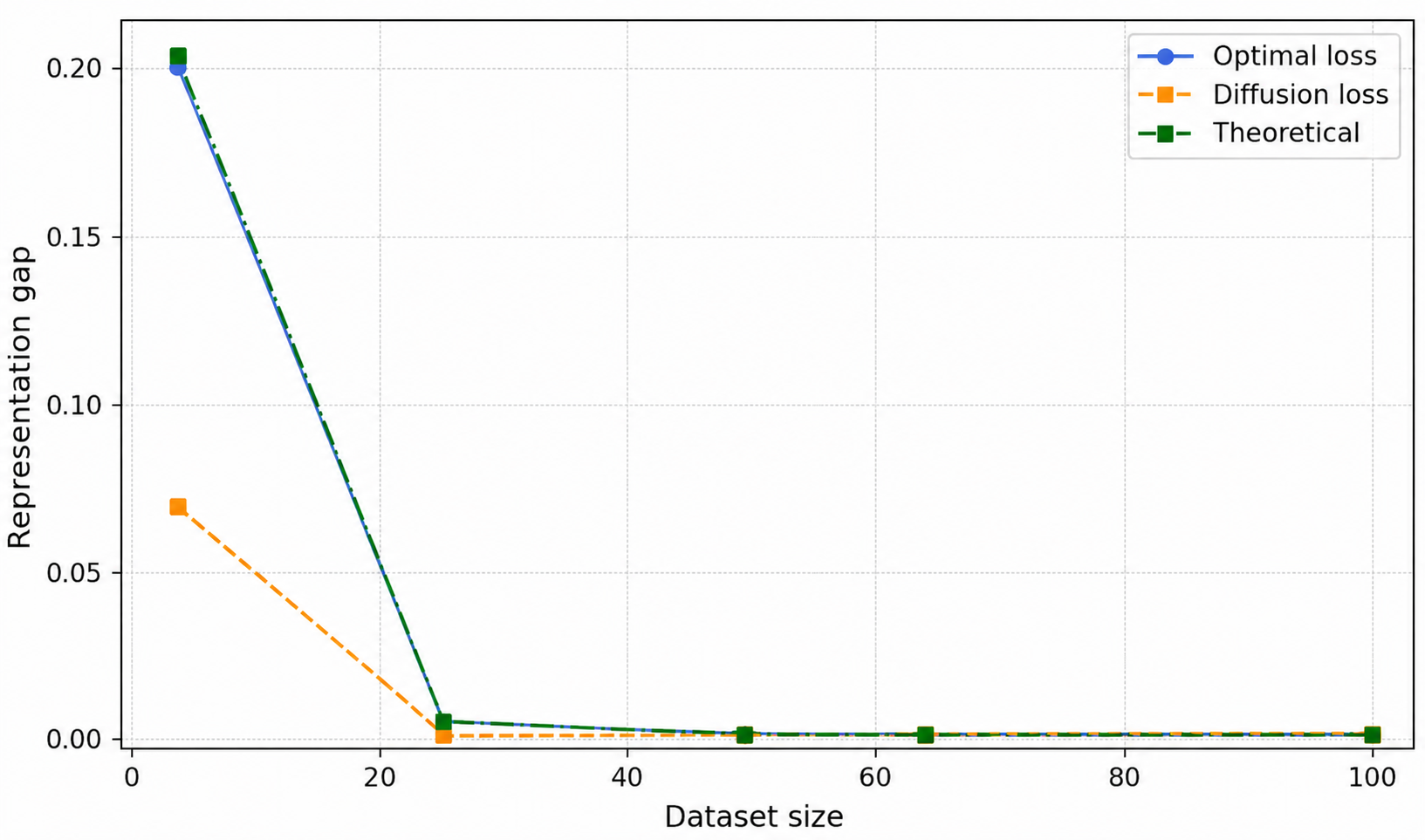} \\
    (b) Wave \\
\end{tabular}
\caption{Asymptotic behavior of the representation gap across the two datasets of Figure \ref{fig:overfitting_example}. The $x$-axis corresponds to the number of training points $n$, and the $y$-axis corresponds to the representation gap. We plot the theoretical loss in Eq. \ref{eq:efficiency_nonconditional} (green), the optimal representation gap $\cR^*_n$ (blue) and the empirical representation gap $\cR_n(\Omega,\Omega_f)$ computed from a diffusion model $f$ trained on an optimal dataset $\bD$ (orange). }
\label{fig:representation_gap}
\vspace{-30pt}
\end{wrapfigure}

\noindent \textit{Swiss roll} and \textit{Deformed sphere} are popular 2-dimensional manifolds embedded into a 3-dimensional ambient spaces \citep{jacobsen2025staying}.

\noindent {\em Real Data:} Following \citet{pope_intrinsic_2021}, we also consider the datasets \textit{MNIST} \citep{lecun_gradient-based_1998}, SVHN \citep{netzer2011reading}, CIFAR10 \citep{krizhevsky2009learning}, MSCOCO \citep{lin2014microsoft}, and Tiny-ImageNet \citep{le2015tiny}. which are standard benchmarks for machine learning and intrinsic dimension estimation \citep{pope_intrinsic_2021,ansuini_intrinsic_nodate}. 

\textbf{Architecture.} For the non-conditional task, we use a three-layer MLP \citep{rumelhart_learning_1986} with ReLU activation and 128 hidden units. For the conditional task, we use a 10-layer MLP with SiLU activation \citep{ramachandran_searching_2017}, 128 hidden units, residual connections, and linear embedding for the conditioning. 
Model equivariance is enforced on top of this architecture by input normalization and prediction shift.

\textbf{Training and optimization.} For the synthetic experiments, we use a DDIM diffusion model \citep{song_denoising_2022}, trained with a linear temperature schedule with $T=100$ steps. We use the $\cL_2$ loss defined on the ambient space $\cY$. The models are trained with the Adam optimizer \citep{kingma_adam_2017} for 50000 steps, with learning rate $\lambda=10^{-3}$. 

\textbf{Metric.} To estimate the representation gap, we sample 1000 points from the trained model and 1000 points uniformly from a hold-out set on $\Omega$. We then approximate Eq. \ref{eq:representation_gap} by computing the average minimum distance between these two clouds of points.

\textbf{Training point selection.} We experiment with both \textit{i.i.d.} sampling and optimally diverse sampling (see Section \ref{sec:theory}). For simple synthetic datasets such as Hypercube and HyperSphere, the optimal samples admit an analytic expression that we can use. For the remaining synthetic datasets, we estimate the optimal samples empirically by minimizing Eq. \ref{eq:representation_gap}. Although this optimization problem is NP-hard \citep{aloise2009np}, standard optimization algorithms such as K-means++ \citep{arthur2006k} and discrete Lloyd on-manifold centroid snapping \citep{lloyd1982least} were sufficient in practice. 

\textbf{Intrinsic dimension estimation.} In order to estimate the intrinsic dimension of a task, we evaluate the random representation gap $\cR_n$ for several values of $n$, fit a linear model to the resulting points $(\log(n), \log(\cR_n))$, and extract the slope, which equals $-2/d$ according to Theorem \ref{prop:efficiency_nonconditional}. We can alternatively use the optimal representation gap $\cR_n^*$, in which case the optimal samples are selected by empirically minimizing Eq. \ref{eq:representation_gap} as described above. We observed that this second estimator converged faster in practice, often with as few as $n=50$ samples. 
As a corollary of Theorem \ref{prop:efficiency_informal}, Eq. \ref{eq:simplified_asymptotic_gap} naturally defines an estimator of the intrinsic dimension of the data manifold $\Omega$. To this end, we estimate the representation gap directly from the training points $\bD$ instead of using the model predictions $\Omega_f$. We repeat intrinsic dimension estimation over 5 random seeds and report the mean and standard deviation.

\subsection{Validating Theorems 2 and 3}\label{sec:validating}

Figures \ref{fig:illustrative_example} and \ref{fig:overfitting_example} highlight two distinct regimes. Non-equivariant models converge toward the training dataset $\Omega_f=\bD$, while equivariant models converge toward the virtually augmented dataset $\Omega_f=G(\bD)$. This behavior is consistently observed across different geometries and experimentally validates the claim of Theorem \ref{prop:virtual_augmentation}.

To further validate the asymptotic formula of Theorem \ref{prop:efficiency_nonconditional}, we evaluate the representation gap on several manifolds $\Omega$. We report the results in Figures \ref{fig:sphere_representation_gap} and \ref{fig:representation_gap}. Across all datasets, the empirical representation gap closely follows the predicted asymptotic scaling. Moreover, convergence toward the asymptotic regime occurs rapidly, in some cases with as few as $n=50$ samples. This suggests that the asymptotic analysis remains relevant in practical settings, including applications where only limited amounts of training data are available.

\subsection{Intrinsic dimension estimation}\label{sec:dimension_estimation}

\begin{wraptable}{r}{0.45\textwidth}
\vspace{-14pt}
\centering
\caption{Estimated intrinsic dimension on synthetic and real-world datasets. 
On synthetic data, both $\cR_n$ and $\cR_n^\star$ recover the true dimension accurately. 
On real-world datasets, our estimates are consistent with estimates from prior work \citep{pope_intrinsic_2021,ansuini_intrinsic_nodate}.}
\label{tab:dimension_all}
\scriptsize
\setlength{\tabcolsep}{3pt}

\begin{tabular}{lccc}
\toprule
\multicolumn{4}{c}{\textbf{Synthetic datasets}} \\
\midrule
Dataset & True & $\cR_n$ & $\cR_n^\star$ \\
\midrule
Cube ($d=1$)    & 1 & $0.95 \pm 0.01$ & $1.07 \pm 0.01$ \\
Cube ($d=5$)    & 5 & $4.93 \pm 0.03$ & $5.06 \pm 0.09$ \\
Sphere ($d=1$)  & 1 & $1.00 \pm 0.02$ & $1.08 \pm 0.01$ \\
Sphere ($d=5$)  & 5 & $4.91 \pm 0.04$ & $5.13 \pm 0.02$ \\
Swiss roll      & 2 & $1.87 \pm 0.01$ & $2.19 \pm 0.03$ \\
Def. sphere     & 2 & $1.87 \pm <0.01$ & $1.94 \pm 0.02$ \\
\midrule
\multicolumn{4}{c}{\textbf{Real-world datasets}} \\
\midrule
Dataset & $\cR_n$ & Pope et al. & Ansuini et al. \\
\midrule
MNIST    & 13 & 7--13  & $\sim$12.5 \\
SVHN     & 13 & 9--19  & -- \\
CIFAR-10 & 18 & 13--26 & -- \\
MSCOCO   & 19 & 22--36 & -- \\
ImageNet & 21 & 26--43 & -- \\
\bottomrule
\end{tabular}
\vspace{-8pt}
\end{wraptable}

As a corollary of Theorem \ref{prop:efficiency_informal}, Eq. \ref{eq:simplified_asymptotic_gap} naturally defines an estimator of the intrinsic dimension of the data manifold $\Omega$. Using the experimental settings described above, we estimate the intrinsic dimension of several synthetic datasets and report the results in Table \ref{tab:dimension_all}. The estimated dimensions closely match the ground-truth manifold dimensions across different geometries and ambient dimensions, which empirically validates the estimator in controlled settings.

We further evaluate the estimator on real-world datasets using \textit{i.i.d.} sampling. As shown in Table \ref{tab:dimension_all}, our estimates are consistent with prior intrinsic dimension estimators \citep{pope_intrinsic_2021,ansuini_intrinsic_nodate}. We observe a progressive increase in intrinsic dimension with dataset complexity: MNIST and SVHN, which both correspond to digit recognition tasks, exhibit similar dimensions, while CIFAR-10 has a larger estimated dimension consistent with its greater visual diversity. MSCOCO and Tiny ImageNet, which contain a broader range of semantic classes and visual structures, exhibit the highest intrinsic dimensions.

\section{Conclusion}

This work introduces a new metric -- the representation gap --, that characterizes neural network generalization from a geometric perspective. We provide a detailed asymptotic analysis of this representation gap in two important settings: non-conditional generative modeling and supervised prediction, under both \textit{i.i.d.} and optimally diverse sampling regimes. We show that the asymptotic scaling of the representation gap is governed by a single parameter, the intrinsic dimension of the task, and relates naturally to classical notions of generalization. In particular, we show how standard machine learning techniques, such as equivariant architectures, reduce this intrinsic dimension, thereby provably improving asymptotic generalization. We validate our theoretical results both on controlled synthetic environments and real-world datasets. Our results suggest that intrinsic dimension may serve as a unifying geometric principle for understanding generalization, and could be leveraged to inform network architecture and training pipeline design in a principled manner. 

More generally, our work suggests shifting the focus from the parameter space of neural networks and the statistical properties of training algorithms to the geometry of the data and prediction spaces. Indeed, global properties of trained models -- such as memorization, equivariance, or minimal-norm interpolation -- induce geometric structure in the prediction space, thereby reducing generalization to a geometric problem. Beyond generalization itself, this perspective provides a principled way to characterize the information contained within a neural network through the geometry of its prediction space. Important applications include architecture-agnostic knowledge transfer and collaborative learning. More fundamentally, we believe that the geometry of the prediction space may provide a principled notion of information based on teachability, opening the way toward a geometric understanding of real-world task uncertainty and the intrinsic limits of statistical learning.

\appendix

\bibliographystyle{plainnat}
\bibliography{bibliography}

\newpage

\section{Notations}\label{appsec:notations}

\subsection{Task and geometry}

\textbf{Manifold.}
We consider a supervised task, with input space $\cX\subset\cR^{d_\cX}$ and target space $\cY\subset\cR^{d_\cY}$. We assume that the observations $(x,y)$ belong to a subset $\Omega\subset\cX\times\cY$, which models the structure of the task and its underlying symmetries. Following the manifold hypothesis \citep{bengio_representation_2013}, we assume that $\Omega$ corresponds to a low-dimensional manifold embedded in the ambient space $\cX\times\cY$. More precisely, we assume throughout this work that $\Omega$ is a Riemannian manifold \citep{lee_riemannian_2006}, and we denote by $d_\Omega$ its dimension. We denote by $\Omega_\cX$ and $\Omega_\cY$ the projections of $\Omega$ onto $\cX$ and $\cY$ respectively.

\textbf{Quotient manifold.} Machine learning tasks typically feature symmetries (\textit{e.g.}, translation invariance in image classification or rotational equivariance in molecular modeling), which are reflected in the structure of the manifold $\Omega$. In this work, we focus on tasks exhibiting symmetries described by a group $G$ acting on the manifold $\Omega$. We denote by $G(y)=\{g(y)\mid g\in G\}$ the orbit of a point $y\in\Omega$ under the action of $G$, and by $G(E)=\cup_{y\in E}G(y)$ the orbit of a set $E$. We will assume that $G$ is a Lie group acting by isometries on $\Omega$. Under this assumption, we can define the quotient manifold $\Omega/G$ \citep{lee_riemannian_2006}, and denote by $d_{\Omega/G}$ its dimension.

\textbf{Metric.}
We denote by $\ell:\Omega\times\Omega\to\bR_+$ a non-negative cost function on $\Omega$. Unless stated otherwise, $\ell$ corresponds to the squared geodesic distance induced by the Riemannian metric on $\Omega$ \citep{peyre2019computational}. We denote by $\ell_\cX$ and $\ell_\cY$ the corresponding cost functions induced by projection onto $\cX$ and $\cY$. For a point $y\in\Omega$ and a subset $E\subset\Omega$, we define $\ell(y,E)=\inf_{y'\in E}\ell(y,y')$. We denote by $\ell_{\Omega/G}$ the quotient metric induced by $\ell$ on $\Omega/G$ \citep{lee_riemannian_2006}.

\subsection{Data and model}

\textbf{Data.} We assume access to a dataset $\bD\subset\Omega$ composed of $n$ observations. For a given input $x\in\cX$, we define the conditional manifold $\Omega_x=\{y\mid (x,y)\in\Omega\}$, which corresponds to the set of admissible targets associated with the input $x$. Similarly, we define the conditional dataset $\bD_x=\{y\in\bD_\cY \mid (x,y)\in\bD\}$. We denote by $\bD_\cX=\{x\mid (x,y)\in\bD\}$ and $\bD_\cY=\{y\mid (x,y)\in\bD\}$ the sets of inputs and targets appearing in $\bD$ respectively. In deterministic supervised settings, each input $x\in\Omega_\cX$ is associated with a unique target $y(x)\in\Omega_\cY$, so that the manifold $\Omega$ can be identified with the graph of a function $y:\Omega_\cX\to\Omega_\cY$.

\textbf{Model.} We consider neural networks $f_\theta$ belonging to a parametric family $\cF_\Theta\subset\cF(\cX,\cY)$. When there is no ambiguity, we will simplify the notation and denote the neural networks by $f$. A model $f$ is said to be equivariant under the action of a group $G$ if we have $g(f(x))=f(g(x))$ for all $x\in\cX$ and $g\in G$. 

\subsection{Probability and asymptotic notation}

\textbf{Probability.} We denote by $\bP$ a probability distribution supported on $\Omega$, and by $p$ its density. We denote by $\delta_x$ the Dirac distribution centered at a point $x$. 
We denote by $\cN(\mu,\sigma^2)$ the Gaussian distribution with mean $\mu$ and variance $\sigma^2$, and by $\cN(y\mid\mu,\sigma^2)$ the evaluation of its density at a point $y$. We denote by $\bI[E]$ the indicator function of a set $E$. For a finite set $E$, we denote by $|E|$ its cardinality. If $E$ is measurable, $|E|$ denotes its Lebesgue measure, and $\mathring{E}$ its interior.

\textbf{Asymptotic notation.}
We denote by $a_n\sim b_n$ the deterministic asymptotic equivalence $\frac{a_n}{b_n}\to 1$. 
Similarly, we write $X_n\sim_{\bP} a_n$ when $\frac{X_n}{a_n}\to 1$ in probability. 
We use the standard notations $\to_d$, $\to_{L^1}$ and $\to_{\bP}$ for convergence in distribution, convergence in $L^1$, and convergence in probability respectively.

\subsection{Diffusion model}

We will focus on Denoising Diffusion Implicit Models (DDIM) diffusion models \citep{song_denoising_2022}. These models are trained to reverse a stochastic forward diffusion process that incrementally adds Gaussian noise to the data distribution while shrinking data points toward the origin. Noise addition is governed by a noise schedule $\alpha_t$, with $t\in[0,T]$. At each schedule step, the noised distribution can be written $\pi_t(y)=\frac{1}{|\bD|}\sum_{z\in\bD}\cN(y|\sqrt{\alpha_t}z, (1-\alpha_t)I)$. In particular, $\pi_0=\frac{1}{|\bD|}\sum_{z\in\bD}\delta_z$ recovers the empirical data distribution and $\pi_T=\cN(0,I)$ is an isotropic Gaussian distribution. In this context, the model $f_\theta:\cY\times\bR\rightarrow\cY$ is trained to approximate the score function $s_t=\nabla \log\pi_t$ using the loss 
\begin{equation}\label{appeq:diffusion_training}
    \cL(\theta)=\bE_{t\sim \bU[0,T],y_0\sim\pi_0,\eta\sim\cN(0,I)}\|f_\theta(\sqrt{\alpha_t}y_0+\sqrt{1-\alpha_t}\eta,t)-\eta\|_2^2\;.
\end{equation}
At sampling time, an initial point $y_T\sim\cN(0,I)$ is sampled and then updated using the deterministic flow 
\begin{equation}\label{appeq:diffusion_sampling}
    \dot{y}_t=-\gamma_t(y_t+s_t(y_t))\;,
\end{equation}
where $t$ goes backward from $T$ to $0$. The output of the model corresponds to the endpoints of these trajectories.

These equations can be generalized to the conditional case. In particular, the model $f_\theta:\cX\times\cY\times\bR\rightarrow\cY$ is trained using the loss
\begin{equation}\label{appeq:conditional_diffusion_training}
    \cL(\theta)=\bE_{t\sim \bU[0,T],(x_0,y_0)\sim\pi_0,\eta\sim\cN(0,I)}\|f_\theta(x_0,\sqrt{\alpha_t}y_0+\sqrt{1-\alpha_t}\eta,t)-\eta\|_2^2\;.
\end{equation}

\subsection{Prediction space}

We define the prediction space $\Omega_f$ of a model $f$ as follows.

\begin{definition}[Prediction space]\label{appdef:prediction_manifold}
    Let $f$ denote a potentially non-deterministic neural network. For each $x\in\cX$, we denote by $\cO_f(x)\subset\cY$ the set of outputs that can be generated by $f$ when conditioned on $x$. We define the prediction space of the model $f$ by 
    \begin{equation}
        \Omega_f = \{(x,y)\;|\; x\in\Omega_\cX, y\in \cO_f(x) \;\}  \;.
    \end{equation}
\end{definition}

In particular, if $f$ is a non-conditional DDIM diffusion model, $\Omega_f$ is the set of endpoints reachable by the reverse flow described by Eq. \ref{appeq:diffusion_sampling}: 
\begin{equation}
    \Omega_f=\{y_0 \;|\; y_T\sim\cN(0,I)\;,\; y_t \text{ solves Eq. } \ref{appeq:diffusion_sampling} \}\; \;.
\end{equation}

\section{Preliminaries}\label{appsec:preliminaries}

\subsection{Asymptotic quantization theory}

A natural way to compare a manifold $\Omega$ with a discrete approximation $\{z_k\}_{k\in\bn}$ is to use the quantization error
\begin{equation}\label{appeq:quantization_error}
    \int_\Omega \min_{k\in\bn}\ell(y,z_k)\,p(y)\,dy.
\end{equation}

This quantity corresponds to a particular case of the Wasserstein distance \citep{peyre2019computational} (see Section \ref{appsec:link_wasserstein}). 
Given a budget of $n$ points, we are often interested in the best approximation achievable by a discrete set of size $n$. This quantity, known as the optimal quantization error or optimal quantization risk \citep{graf2007foundations}, is defined as
\begin{equation}\label{appeq:optimal_quantization_error}
    \cR_n(\bP)=\inf_{z\in\cY^n}\int_\cY \min_{k\in\bn}\ell(y,z_k)\,p(y)\,dy.
\end{equation}

A central tool of our analysis is Zador's theorem \citep{zador1982asymptotic}, a powerful result characterizing the asymptotic distribution of the centroids resulting from optimal quantization. Intuitively, this theorem describes how well a continuous manifold can be approximated by a finite number of representative points, and how this approximation scales with the intrinsic dimension of the underlying space. We first recall the Euclidean version (see \citet{graf2008distortion}, Eq. 2.3, or \citet{iacobelli2016asymptotic}, Theorem 1.3, for a more general version).

\begin{theorem}[Zador theorem]\label{app_prop:zador}
    Let $\bP = p \: \dy$ be a Lebesgue-dominated probability measure on a compact subset $\cY$ of $\bR^d$. Then the optimal quantization error $\cR_n(\bP)$ satisfies
    \begin{equation}\label{appeq:zador}
        \cR_n(\bP) \underset{n \to \infty}{\sim} J_d^* \;\left(\int_\cY p(y)^{\frac{d}{d+2}}\dy\right)^{\frac{d+2}{d}} \frac{1}{n^{2/d}}\:,
    \end{equation}
    where $J_d^*$ is the asymptotic optimal quantization error for the uniform distribution 
    \begin{equation}\label{appeq:optimal_universal_constant}
        J_d^* = \inf_n n^{2 / d} \cR_n(\mathcal{U}([0,1]^d))\;.
    \end{equation}
\end{theorem}

The constant $J_d^*$ can be computed for simple cases ($J^*_1=\frac{1}{12}$ and $J^*_2=\frac{5}{18\sqrt{3}}$ \citep{newman1982hexagon}) and can be approximated for large $d$ by $J^*_d \sim \frac{d}{2\pi e}$ \citep{pages2003optimal, graf2007foundations}.

A generalization of Zador theorem to arbitrary manifolds has been proposed in \citet{gruber_optimal_nodate}, which we report below (see Theorem 2 in this reference for a stronger result). In this case, the asymptotic behavior depends only on the intrinsic geometry of the manifold rather than on the ambient Euclidean space.

\begin{theorem}[Zador theorem on manifold]\label{app_prop:zador_manifold}
    Let $d=d_\Omega$ denote the dimension of the manifold $\Omega$. Then there exists a constant $J_d^*$ depending only on $d$  and the metric $\ell$ such that for all $E\subset\Omega$ compact and measurable with $|E|>0$ and all $p:E\to\bR^+$ continuous, we have
    \begin{equation}\label{appeq:zador_manifold}
        \inf_{z\in\cY^n} \int_E \min_{k\in\bn} \ell(y,z_k) p(y)\dy \underset{n \to \infty}{\sim} J_d^* \;\left(\int_E p(y)^{\frac{d}{d+2}}\dy\right)^{\frac{d+2}{d}} \frac{1}{n^{2/d}}\:.
    \end{equation}
\end{theorem}

For an arbitrary manifold $\Omega$, we define the density-dependent volume functional appearing in Eq. \ref{appeq:zador} and Eq. \ref{appeq:zador_manifold} by
\begin{equation}\label{appeq:optimal_volume}
    \cV_d^*(p)=\left(\int_\Omega p(y)^{\frac{d}{d+2}}\dy\right)^{\frac{d+2}{d}}\;.  
\end{equation}

\subsection{Point process theory}

The quantization error in Eq. \ref{appeq:quantization_error} is closely related to the nearest neighbor distance, which has been extensively studied by point process theory \citep{biau2015lectures}. 
Point process theory provides powerful tools to study the asymptotic geometry of random point configurations, including the nearest neighbor distance $\min_{k\in\bn} \ell(Y,Z_k)$ where $Y$ and $Z_k$ are i.i.d variables following the data distribution $\bP$. An important result concerns the convergence rate of the expected nearest neighbor distance, which we state below in the Euclidean setting (see Theorem 2.3 p.20 in \citet{biau2015lectures}). This result has also been extended to Riemannian manifolds $\Omega$ and other metrics $\ell$ (see for instance Theorem 2 in \citet{costa2006determining}).

\begin{theorem}[Convergence rate of the nearest neighbor distance]\label{app_prop:nearest_neighbor}
    Assume $\Omega=[0,1]^d$, $d>2$, and $\ell(y,z)=\|y-z\|_2^2$ denotes the squared Euclidean distance. Let $Y,Z_1,\ldots Z_n$ denote \textit{i.i.d.} random variables following a distribution $\bP$  with density $p$, $\bD=\{Z_1,\ldots,Z_n\}$ denote the resulting \textit{i.i.d.} dataset, and $\cR(\bD)=\bE_Y \left[\;\min_{k\in\bn} \|Y-Z_k\|_2^2\;\right]$ denote the resulting quantization error. Then
    \begin{equation}\label{appeq:nearest_neighbor_distance_rate}
        \bE_{\bD}\cR(\bD) \underset{n \to \infty}{\sim} J_d \int_{[0,1]^d} p(y)^{(d-2)/d}\dy\;\frac{1}{n^{2/d}}\;,
    \end{equation}
    where 
    \begin{equation}\label{appeq:random_universal_constant}
        J_d=\frac{1}{\pi}\Gamma\left(\frac{2}{d}+1\right)\Gamma\left(\frac{d}{2}+1\right)^{2/d}   
    \end{equation}
\end{theorem}

For an arbitrary manifold $\Omega$, we define the density-dependent volume functional appearing in Eq. \ref{appeq:nearest_neighbor_distance_rate} by
\begin{equation}\label{appeq:random_volume}
\cV_d(p)=\int_\Omega p(x)^{(d-2)/d}\dx\;.
\end{equation} 

Interestingly, both optimal quantization and random nearest-neighbor approximation exhibit the same asymptotic scaling law in $n^{-2/d_\Omega}$. In both cases, the asymptotic behavior decomposes into: (i) a universal geometric constant, independent of the data distribution ($J_d$, Eq. \ref{appeq:random_universal_constant}, or $J_d^*$, Eq. \ref{appeq:optimal_universal_constant}), and (ii) a density-dependent volume functional ($\cV_d$, Eq. \ref{appeq:random_volume}, or $\cV_d^*$, Eq. \ref{appeq:optimal_volume}). Furthermore, both volume functionals reduce to the same geometric scaling $\cV_d(p)=\cV_d^*(p)=|\Omega|^{2/d}$ for a uniform density $p$. Our work unifies these two asymptotic regimes (Theorems \ref{app_prop:zador_manifold} and \ref{app_prop:nearest_neighbor}) under a common geometric framework.

Despite these similarities, Theorem \ref{app_prop:nearest_neighbor} and Zador's theorem describe fundamentally different types of asymptotic results. Theorem \ref{app_prop:nearest_neighbor} characterizes the quantization error averaged over all \textit{i.i.d.} datasets of size $n$, whereas Zador's theorem characterizes the quantization error of a specific (optimal) point configuration. From a learning perspective, results in expectation are insufficient to characterize the behavior of a model trained on a specific dataset. One of the main contributions of our work is to bridge these two asymptotic regimes by extending Theorem \ref{app_prop:nearest_neighbor} to a convergence result in probability.

\section{Non-conditional tasks}\label{appsec:non_conditional}

\subsection{Memorizing networks and representation gap}

Let us first consider the case of a non-conditional prediction task. This setting corresponds to unconditional generative modeling, where the goal is to learn a probability distribution supported on $\Omega\subset\cY$ that captures the geometric structure of the data manifold.

Popular approaches for generative modeling include diffusion models \citep{ho_denoising_nodate,song_denoising_2022}, Variational Auto Encoders (VAE) \citep{kingma_auto-encoding_2022}, Generative Adversarial Networks (GAN) \citep{goodfellow_generative_2014} or normalizing flows \citep{rezende_variational_2016}. Among them, diffusion models can be shown to converge toward the empirical distribution $\frac{1}{|\bD|}\sum_{y\in\bD}\delta_y$ when they minimize their training objective \citep{song_generative_nodate}. 

We will focus on this class of models hereafter. In this case, the empirical distribution corresponds to the prediction space $\Omega_f$ learned by the model $f$, which can be seen as a discrete approximation of $\Omega$. We can compare this discrete prediction space $\Omega_f$ to $\Omega$ using the quantization error \citep{zador1982asymptotic}. This metric can be extended in the more general case where $\Omega_f$ may be continuous. We will refer to this distance as the \textit{representation gap}.

\begin{definition}[Representation gap]\label{appdef:representation_gap}
    Let $\Omega$ denote the data manifold and $\Omega_f$ denote the model's prediction space. We define the representation gap as follows:
    \begin{equation}\label{appeq:representation_gap}
        \cR(\Omega,\Omega_f)=\int_\Omega\inf_{z\in\Omega_f}\ell(y,z)\;p(y)\;\dy\;.
    \end{equation}
\end{definition}

The representation gap $\cR(\Omega,\Omega_f)$ depends on how the dataset $\bD$ is sampled from $\Omega$. It is typical to assume that $\bD$ is sampled \textit{i.i.d.} from the data distribution $p$, and we denote \textit{random representation gap} the corresponding quantity.

\begin{definition}[Random representation gap]\label{appdef:random_representation_gap}
    Let $\Omega$ denote the data manifold, $\bD$ a dataset of size $n$ sampled \textit{i.i.d.} from p, and $\Omega_f=\Omega_f(\bD)$ denote the prediction space of a model $f$ trained on $\bD$. We define the random representation gap as the random variable
    \begin{equation*}
        \cR_n=\int_\Omega\inf_{z\in\Omega_f}\ell(y,z)\;p(y)\;\dy\;.
    \end{equation*}
\end{definition}

Although theoretical analyses of generalization typically assume that training and test data are \textit{i.i.d.} \citep{shalev2014understanding}, the validity of this hypothesis has been questioned in the literature \citep{mohri_rademacher_nodate}.
Indeed, datasets are often collected to cover the diversity of the task, modulo its known invariants \citep{torralba2011unbiased}. Motivated by this observation, we also consider the setting where $\bD$ is optimally diverse, i.e. minimizes the representation gap, and we denote \textit{optimal representation gap} the corresponding quantity.

\begin{definition}[Optimal representation gap]\label{appdef:optimal_representation_gap}
    Let $\Omega$ denote the data manifold. For each dataset $\bD\subset\Omega$, we denote by $\Omega_f=\Omega_f(\bD)$ the prediction space of a model $f$ trained on $\bD$. We define the optimal representation gap as  
    \begin{equation*}
        \cR_n^*=\inf_{\bD\subset\Omega,\; |\bD|=n}\cR(\Omega,\Omega_f(\bD)).
    \end{equation*}
\end{definition}

Note that the random representation gap $\cR_n$ is a random variable while the optimal representation gap $\cR_n^*$ is a scalar value. Both quantities are  notoriously difficult to study, even in the discrete case \citep{graf2007foundations}. However, they become amenable to analysis in the asymptotic regime.

\subsection{Asymptotic representation gap in the Euclidean setting}

The asymptotic scaling of the representation gap naturally characterizes how efficiently the geometry of a task can be learned from a finite dataset.

\begin{prop}[Optimal representation gap]\label{app_prop:optimal_efficiency_general}
      Let us assume that $\Omega$ is Lebesgue-measurable with positive measure. Then, the optimal representation gap of a diffusion model $f$ minimizing its training objective \ref{eq:diffusion_training} on a training dataset of size $n$ is 
      \begin{equation}
        \cR_n^* \underset{n\rightarrow+\infty}{\scalebox{1.5}{$\sim$}} \frac{J_d^* \cV_d^*(p)}{n^{2/d}}\;.   
      \end{equation}
\end{prop}

\begin{proof}
This is a corollary of Zador Theorem \ref{app_prop:zador}.
\end{proof}

This result is remarkable, since it provides an asymptotic equivalent of the representation gap as the dataset size $n$ grows to infinity. Most notably, the leading constant depends on the geometry of $\Omega$ only via a volume term $\cV_d^*(p)$.

\subsection{Asymptotic representation gap under the manifold hypothesis}

It is possible to extend this result when $\Omega$ is a low-dimensional manifold of the target space $\cY$. This setting is interesting because it captures the structure of the observation manifold $\Omega$: even though the observation could a priori be an arbitrary point of $\cY$, it is in effect restricted to the subspace $\Omega$. 

\begin{prop}[Optimal representation gap under the manifold hypothesis]\label{app_prop:optimal_efficiency_manifold}
      Assume that $\Omega$ is a compact Riemannian $d_\Omega$-manifold. Then the optimal representation gap of a diffusion model $f$ minimizing its training objective \ref{eq:diffusion_training} on a training dataset $\bD$ of size $n$ satisfies 
      \begin{equation}
        \cR_n^* \underset{n\rightarrow+\infty}{\scalebox{1.5}{$\sim$}} \frac{J_{d_\Omega}^* \; \cV_{d_\Omega}^*(p)}{n^{2/d_\Omega}}\;.      
      \end{equation}
\end{prop}

\begin{proof}
    This is a corollary of Theorem 2 in \cite{gruber_optimal_nodate} (see also Theorem \ref{app_prop:zador_manifold}). The assumptions of the theorem are satisfied, since $\Omega$ is compact and the squared geodesic distance satisfies the required growth condition. We denote by $J_d^*$ the corresponding asymptotic constant, consistently with the Euclidean formulation of Zador's theorem.
\end{proof}

This asymptotic evolution is similar to the general case described in Proposition \ref{app_prop:optimal_efficiency_general}, but leverages the structure of $\Omega$ via the lower dimension $d_{\Omega}$. Note that we recover Proposition \ref{app_prop:optimal_efficiency_general} when $\Omega$ has positive measure in $\cY$. Again, we highlight that the leading constant depends on the geometry of $\Omega$ only via a volume term $\cV_{d_\Omega}^*(p)$. Moreover, it can be proved that the optimal data placement for $\bD$ is uniformly distributed in $\Omega$ when $p$ is uniform (cf. point 2.82 in \cite{gruber_optimal_nodate}). 

Leveraging results from point process theory, we describe the asymptotic random representation gap in the following Proposition. 

\begin{prop}[Random representation gap under the manifold hypothesis]\label{app_prop:random_efficiency_manifold}
      Assume that $\Omega$ is a compact $d_{\Omega}$-dimensional Riemannian manifold without boundary and $p$ is continuous and strictly positive on $\Omega$. Then the random representation gap of a diffusion model $f$ minimizing its training objective \ref{eq:diffusion_training} on a training dataset $\bD$ of size $n$ satisfies 
      \begin{equation}
        \cR_n \sim_\bP \frac{J_{d_\Omega} \cV_{d_\Omega}(p)}{n^{2/d_\Omega}}\;,      
      \end{equation}
      where we have defined $J_d=\frac{1}{\pi}\Gamma(1+\frac{d}{2})^{2/d}\Gamma(1+\frac{2}{d})$ and $\cV_d(p)=\int_\Omega p(x)^{(d-2)/d}\dx$.
\end{prop}

\begin{proof}
    The idea is to extend the proof of Theorem 2.3 in \citet{biau2015lectures} to the manifold setting. Intuitively, the argument relies on the following three observations: (i) the random representation gap $\cR_n$ can be written as the expectation over $Y$ of the random variable
    $A_n(Y)=n^{2/d}\min_{z\in\bD}\ell(Y,z)$, up to normalization factors; (ii) for fixed $y\in\Omega$, the asymptotic behavior of $A_n(y)$ is entirely characterized by its tail distribution $\bP(A_n(y)>t)$; (iii) under the manifold hypothesis, one can show that $\bP(A_n(y)>t)\rightarrow\exp(-p(y)V_d t^{d/2})$, which gives the result after integration over $t$ and $y$.
    
    We now prove this result formally. We write $d=d_\Omega$ for simplicity. Let $z_1,\ldots,z_n\sim p$, $\bD=\{z_1,\ldots,z_n\}$, and define $A_n(y)=n^{2/d}\min_{1\leq j \leq n} \ell(y,z_j)$. Let $A(y)$ denote a random variable such that $\mathbb{P}(A(y) > t)=\exp(-p(y)V_d t^{d/2})$, where $V_d$ denotes the volume of the unit-ball in $\mathbb{R}^d$. 
    
    We know that $\mathbb{P}(A_n(y) > t)=\left(1-\mathbb{P}\left(B\left(y,\sqrt{t}\; n^{-1/d}\right)\right)\right)^n$, where $B\left(y,\sqrt{t}\; n^{-1/d}\right)$ denote the ball of radius $\sqrt{t}\; n^{-1/d}$ centered on $y$ in the manifold $\Omega$.
    Moreover, we know that the volume of a ball $B(y,\varepsilon)$ in a manifold $\Omega$ of dimension $d$ can be approximated by $|B(y,\varepsilon)|=V_d\varepsilon^d + o(\varepsilon^d)$ (see for instance Eq. 1 in \citet{barilari2018volume}). By continuity of $p$, we then have $\bP(B(y,\varepsilon))=p(y) V_d\varepsilon^d+ o(\varepsilon^d)$. Therefore, we deduce
    \begin{equation*}
        \mathbb{P}(A_n(y) > t)=\left(1-\mathbb{P}\left(B\left(y,\frac{\sqrt{t}}{n^{1/d}}\right)\right)\right)^n
        \underset{n\rightarrow\infty}{\rightarrow}\; \exp(-p(y)V_d t^{d/2})=\mathbb{P}(A(y)>t)\;.
    \end{equation*}
    In particular, $A_n(y)\rightarrow_d A(y)$ in distribution. Moreover, since $\mathbb{P}(B(y,r))=p(y)V_d\;r^d+o(r^d)$, there exists $c(y)>0$ such that for sufficiently large $n$, $\mathbb{P}(A_n(y)>t)\leq \exp(-c(y)t^{d/2})$. Therefore, $(A_n(y))_n$ is uniformly integrable.
    
    We deduce $A_n(y)\rightarrow_{L^1} A(y)$ (Theorem 4.6.3 in \citet{durrett2019probability}), and then $\mathbb{E}_{Y\sim p}[A_n(Y)|\mathbb{D}]\rightarrow_{\mathbb{P}} \mathbb{E}_{Y\sim p}A(Y)$ (Example 4.6.11 in \citet{durrett2019probability}). We conclude by observing $n^{2/d}\mathcal{R}(\Omega,\mathbb{D})=\mathbb{E}_{Y\sim p}[A_n(Y)|\mathbb{D}]$ on one hand. On the other hand, 
    \begin{equation*}
        \mathbb{E}_{Y\sim p}A(Y)=\int_0^\infty\int_\Omega \exp(-p(y)V_d\;t^{d/2})p(y)\dy\dt =J_d \cV_d(p)\;,
    \end{equation*}
    (by change of variable $u=p(y)V_d t^{d/2}$ and using the properties of the $\Gamma$ integral), so that 
    \begin{equation*}
        \mathcal{R}(\Omega,\mathbb{D}) = J_d \cV_d(p)n^{-2/d} + o_{\mathbb{P}}(n^{-2/d})\;.
    \end{equation*}
\end{proof}

Unlike classical nearest-neighbor asymptotics, which characterize the quantization error only in expectation over random datasets, Proposition \ref{app_prop:random_efficiency_manifold} establishes convergence in probability for individual datasets. This distinction is important from a learning perspective, since neural networks are trained on a specific realized dataset rather than on an average over datasets. As such, our result is more directly relevant to the practice of neural network training.

\subsection{Asymptotic representation gap for equivariant models}

In practice, $\cF_\Theta$ has limited expressivity, which introduces biases in the minimizer $f=\argmin_{\theta\in\Theta}\cL(\theta)$. Typically, the architecture of the neural network is chosen so that $f_\theta$ respects the symmetries of $\Omega$, and has therefore higher generalization capabilities. Remarkably, the authors of \cite{kamb_analytic_2025} have shown in the context of diffusion models that these architectural constraints virtually augment the diversity of the dataset $\bD$ via the symmetry group introduced by the architecture. 

The following result is an extension of Theorem B.3 in \cite{kamb_analytic_2025} to general symmetry groups $G$. More precisely, we will focus our attention on Lie groups, which naturally describe many symmetries appearing in neural networks \citep{bronstein_geometric_2021}. They are also used in various fields such as physics, where they reflect the structure and symmetries of many physical systems \citep{gilmore_lie_2006,noauthor_lie_nodate}. This makes them particularly relevant for our purposes.

\begin{prop}[Virtual augmentation of a dataset by a symmetry group]\label{app_prop:virtual_augmentation}
    Let us make the following assumptions
    
    (i) $f$ is a trained diffusion model equivariant to $G$.

    (ii) G is a Lie group acting smoothly on the Riemannian manifold $\Omega$.

    (iii) The minimum $\min_{z\in G(\bD)} \ell(y,z)$ is reached at a unique point $y^*=\argmin_{z\in G(\bD)} \ell(y,z)$ for all $y\in\cY$.
    
    (iv) Let $y_t$ denote the denoising trajectory from the Gaussian distribution $\pi_T$, standard reverse diffusion process $\partial_t y_t=-\gamma_t(y_t + f(y_t,t))$. Assume that $y_t$ converges and $\partial_t y_t$ is bounded for each initial point $y_T$.   
    
    Then, the denoising trajectory ends at $y_0\in G(\bD)$. 

    If we further assume each dataset point $z\in\bD$ is a fixed point of the $f(\cdot,t)$ for all $t$, then each point $z\in G(\bD)$ is a limit point of the reverse diffusion process.
\end{prop}

Proposition \ref{app_prop:virtual_augmentation} essentially states that under mild assumptions, an equivariant diffusion model $f$ will generate samples in the virtually augmented dataset $G(\bD)$. This is because the vision of the model $f$ is blurred due to its equivariance to $G$, so that it cannot distinguish points along the orbits $G(y)$ of the dataset points $y\in\bD$.   

The hypothesis $(i)$ states that the model $f$ is a global minimum of its training objective $\cL$. The hypothesis $(ii)$ restricts our attention to Lie groups $G$, as discussed above. The point $(iii)$ avoids the degenerate case where the initial point $y$ is equidistant to a subset of the orbit of the dataset $G(\bD)$. 
The point $(iv)$ is a slightly relaxed form of a technical assumption introduced in Theorem B.3 of \cite{kamb_analytic_2025}. Finally, the fixed-point hypothesis captures the fact that each point $z\in\bD$ is a local attractor of the score function, since the empirical distribution is discrete in our setting.

The proof of Proposition \ref{app_prop:virtual_augmentation} relies on the following observation:  the score function can be written as an integral over the orbits $G(z)$ of each data point $z\in\bD$, where each point $z$ is weighted by the distribution \begin{equation}
    W_t(z)=\frac{\cN\left(y|\sqrt{\alpha_t}z,(1-\alpha_t)I\right)}{\int_{G(\bD)}\cN\left(y|\sqrt{\alpha_t}z',(1-\alpha_t)I\right)\dz'}\;.
\end{equation} 
In the case where the group $G$ is finite, we can see that $W_t(z)$ acts as a softmax that peaks when $z^*$ as $t\rightarrow0$. In the more general case where $G$ is not finite, we can use a Laplace approximation to show that $W_t(z)$ concentrates the probability mass around the minimizer $z^*$ when $t\rightarrow0$. Therefore, the denoising trajectory is attracted toward the orbit $G(\bD)$.

\begin{lemma}[Laplace approximation]\label{app_prop:laplace_approximation}
    Let G denote a Lie group acting smoothly on $\Omega$, $\alpha_t$ a continuous positive noise schedule satisfying $\alpha_t\rightarrow_{t\rightarrow0}1$, $y\in\cY$ an arbitrary point, $d$ the dimension of $G(\bD)$, and $h$ a bounded continuous non-negative function on $G(\bD)$. Assume that $y$ has a unique closest point $y^*\in \mathring{G(\bD)}$, the interior of the orbit. Define $\beta_t=2\frac{1-\alpha_t}{\alpha_t}$ a temperature scaling. Then, we have
    \begin{align}\label{appeq:laplace_approx}
        \int_{G(\bD)}h(z)\cN\left(y|\sqrt{\alpha_t}z,(1-\alpha_t)I\right)\dz &\underset{t\rightarrow0}{=}  h(y^*)\; e^{-\|y^*-y\|^2/\beta_t}\;(2\pi\beta_t)^{d/2} \\
        &+o\left(\;e^{-\|y^*-y\|^2/\beta_t}\;\beta_t^{d/2}\right)\;.\notag
    \end{align}
\end{lemma}

\begin{proof}
     Let us denote by $I(t)=\int_{G(\bD)}h(z)\cN\left(y|\sqrt{\alpha_t}z,(1-\alpha_t)I\right)\dz$ the left term in Eq. \ref{appeq:laplace_approx}. Informally, the proof of Lemma \ref{app_prop:laplace_approximation} then relies on the two following approximations:
    \begin{equation*}
        I(t)=\int_{G(\bD)}h(z)e^{-\|z-\frac{y}{\alpha_t}\|^2/\beta_t}\dz\ 
        \approx \int_{G(\bD)}h(z)e^{-\|z-y\|^2/\beta_t}\dz 
        \approx h(y^*)e^{-\|y^*-y\|^2/\beta_t}(2\pi\beta_t)^{d/2}\;.
    \end{equation*}
    The first approximation comes from integrating $\|z-\frac{y}{\alpha_t}\|^2=\|z-y\|^2+O(\beta_t)$ over the orbit $G(\bD)$, and the second approximation is an extension of Laplace approximation on measurable subsets of $\bR^d$. It expresses that the Gaussian kernel $e^{-\|z-y\|^2/\beta_t}$ concentrates mass at the minimizer $y^*$, with a curvature term $(2\pi\beta_t)^{d/2}$.

    Let us now prove these two approximations. First observe that
    \begin{equation*}
        \|z-\frac{y}{\alpha_t}\|^2 - \|y^*-\frac{y}{\alpha_t}\|^2=\|z-y\|^2-\|y^*-y\|^2+2\frac{\sqrt{\alpha_t}-1}{\sqrt{\alpha_t}}\langle z-y^*|y\rangle,
    \end{equation*}
    so that by exponentiation and integration, we have
    \begin{equation*}
        \int_{G(\bD)}h(z)e^{-\|z-\frac{y}{\alpha_t}\|^2/\beta_t}\dz=e^{-\|y^*-\frac{y}{\alpha_t}\|^2/\beta_t}  \underbrace{\int_{G(\bD)}h(z) e^{\frac{\sqrt{\alpha_t}}{2(1+\sqrt{\alpha_t})}\langle y^*-z|y\rangle}\;  e^{(\|y^*-y\|^2-\|z-y\|^2)/\beta_t}\dz}_{J(t)}.
    \end{equation*}
    
    The noise schedule $\alpha_t$ is bounded in $[0,1]$, so that $e^{-|\langle y^*-z|y\rangle|}\leq e^{\frac{\sqrt{\alpha_t}}{2(1+\sqrt{\alpha_t})}\langle y^*-z|y\rangle}\leq e^{|\langle y^*-z|y\rangle|}$. Let us define 
    \begin{equation*}
        J_-(t)=\int_{G(\bD)}h(z) \;e^{-|\langle y^*-z|y\rangle|}\;  e^{(\|y^*-y\|^2-\|z-y\|^2)/\beta_t}\dz\;,
    \end{equation*}
    a lower bound of $J(t)$. 

    Then we can apply Corollary 3.4 in \cite{kirwin_higher_2010} to $J(t)$ in order to obtain that $J_-(t)\underset{t\rightarrow0}{=}h(y^*)(2\pi\beta_t)^{d/2}+o(\beta_t^{d/2})$. Indeed, the conditions of this Corollary are met (modulo a change of variable), since $G(\bD)$ is a measurable set which contains $y^*$ as an interior point, $z\mapsto\|y^*-y\|^2-\|z-y\|^2$ is twice differentiable and attains its unique minimum value of 0 at $y^*$, $z\mapsto h(z) e^{-|\langle y^*-z|y\rangle|}$ is a continuous function on $G(\bD)$ evaluating at $h(y^*)$ on $y^*$, and $1/\beta_t\underset{t\rightarrow0}{\rightarrow}+\infty$.

    Likewise, we can also prove that 
    \begin{equation*}
        J_+(t)=\int_{G(\bD)}h(z) \;e^{|\langle y^*-z|y\rangle|}\;  e^{(\|y^*-y\|^2-\|z-y\|^2)/\beta_t}\dz \underset{t\rightarrow0}{=} h(y^*)(2\pi\beta_t)^{d/2}+o(\beta_t^{d/2})\;.
    \end{equation*}
    Therefore, we deduce by squeezing that $J(t)\underset{t\rightarrow0}{=}h(y^*)(2\pi\beta_t)^{d/2}+o(\beta_t^{d/2})$, and we can conclude
    \begin{equation*}
        I(t)=e^{-\|y^*-\frac{y}{\alpha_t}\|^2/\beta_t}J(t) \underset{t\rightarrow0}{=} \; h(y^*)\;e^{-\|y^*-y\|^2/\beta_t}\;(2\pi\beta_t)^{d/2}+o\left(\;e^{-\|y^*-y\|^2/\beta_t}\;\beta_t^{d/2}\right)\;.
    \end{equation*}
\end{proof}

We can now prove Proposition \ref{app_prop:virtual_augmentation}.

\begin{proof}[Proof of Proposition \ref{app_prop:virtual_augmentation}]
    By theorem B.3 in \cite{kamb_analytic_2025}, the score function by the model $f$ can be written
    \begin{equation}\label{appeq:diffusion_approximation}
        f(y_t,t)=-\frac{1}{1-\alpha_t}\frac{\int_{G(\bD)}(y-\sqrt{\alpha_t}z)\cN(y|\sqrt{\alpha_t}z,(1-\alpha_t)I)\dz}{\int_{G(\bD)}\cN(y|\sqrt{\alpha_t}z,(1-\alpha_t)I)\dz}
        = \frac{1}{1-\alpha_t}(y_t-y_t^*)+o\left(\frac{1}{1-\alpha_t}\right), 
    \end{equation}
    where the second equality is a corollary of Lemma \ref{app_prop:laplace_approximation} to be justified later. Then, hypothesis $(iv)$ implies that $\gamma_t f(y_t,t)=\partial_t y_t+\gamma_t y_t$ is bounded, which in turn implies $y_t-y^*_t=(1-\alpha_t) f(y_t,t)\rightarrow0$. Since $y^*_t\in G(\bD)$, which is compact (by hypothesis $(ii)$ and property of Lie groups), and $y_t$ converge (by hypothesis $(iv)$), then $y_t^*$ converge and $\lim_{t\rightarrow0}y_t=\lim_{t\rightarrow0}y_t^*\in G(\bD)$.  

    Therefore, we only need to prove the approximation in Eq. \ref{appeq:diffusion_approximation}. Noting $d$ the dimension of $G(\bD)$, $y_t^*$ the unique minimizer of $\ell(y_t,G(\bD))$ (by hypothesis $(iii)$), and $I(t)=\int_{G(\bD)}(y-\sqrt{\alpha_t}z)\cN(y|\sqrt{\alpha_t}z,(1-\alpha_t)I)\dz$, we can write the following.
    \begin{align*}
        I(t)-(y_t-\sqrt{\alpha_t}y^*_t)(2\pi\beta_t)^{d/2}&=\int_{G(\bD)}(y-\sqrt{\alpha_t}z)\cN(y|\sqrt{\alpha_t}z,(1-\alpha_t)I)\dz \\
        &-\int_{G(\bD)}(y-\sqrt{\alpha_t}y^*)\cN(y|\sqrt{\alpha_t}z,(1-\alpha_t)I)\dz \\
        &= \sqrt{\alpha_t}\int_{G(\bD)}(y^*-z)\cN(y|\sqrt{\alpha_t}z,(1-\alpha_t)I)\dz \\
        \|I(t)-(y_t-\sqrt{\alpha_t}y^*_t)(2\pi\beta_t)^{d/2}\|&\leq \sqrt{\alpha_t}\int_{G(\bD)}\|y^*-z\|\cN(y|\sqrt{\alpha_t}z,(1-\alpha_t)I)\dz
    \end{align*}
    Moreover, the function $z\mapsto\|y^*-z\|$ is bounded, continuous and non-negative on $G(\bD)$, so that the conditions of Lemma \ref{app_prop:laplace_approximation} are satisfied. Therefore, we deduce by bounding that $I(t)-(y_t-\sqrt{\alpha_t}y^*_t)(2\pi\beta_t)^{d/2}=o(\beta_t^{d/2})$, which entails $I(t)=(y_t-y^*_t)(2\pi\beta_t)^{d/2}+o(\beta_t^{d/2})$.

    On the other side, we also deduce from Lemma \ref{app_prop:laplace_approximation} that $\int_{G(\bD)}\cN(y|\sqrt{\alpha_t}z,(1-\alpha_t)I)\dz=(2\pi\beta_t)^{d/2}+o(\beta_t^{d/2})$. Therefore, we have
    \begin{equation*}
        f(y_t,t)=\frac{1}{1-\alpha_t}\frac{(y_t-y_t^*)(2\pi\beta_t)^{d/2}+o(\beta_t^{d/2})}{(2\pi\beta_t)^{d/2}+o(\beta_t^{d/2})}=\frac{1}{1-\alpha_t}(y_t-y_t^*)+o\left(\frac{1}{1-\alpha_t}\right)\;.
    \end{equation*}
    This shows that $\Omega_f \subset G(\bD)$.
    For the reverse inclusion, we will use the assumption that each point $z\in\bD$ is a fixed point of the model $f$. More precisely, assume that $y_t=g(z)\in\bD$ with $g\in G$ and $z\in\bD$. Then $\partial_t y_t=-\gamma_t (g(z) - f(g(z),t)=-\gamma_T g(z-f(z,t))=0$ by equivariance of $f$ and by the fixed point hypothesis. Therefore, a trajectory starting at $y_T\in G(\bD)$ stays at $y_T$, which is hence a limit point.
    
    This establishes $\Omega_f=G(\bD)$ and concludes the proof of Proposition \ref{app_prop:virtual_augmentation}.
\end{proof}

Proposition \ref{app_prop:virtual_augmentation} establishes that an equivariant diffusion model $f$ generates samples in $G(\bD)$. Therefore, we can identify its prediction space $\Omega_f$ with $G(\bD)$.  If the symmetry group $G$ enforced by the architecture is aligned with the symmetries of the manifold $\Omega$, then the effective dimension of the learning problem is reduced from $d_\Omega$ to $d_{\Omega/G}$.

\begin{prop}[Representation gap for an equivariant function]\label{app_prop:efficiency_equivariant} 
      Assume that $\Omega$ is a compact $d_{\Omega}$-dimensional Riemannian manifold without boundary and $p$ is continuous and strictly positive on $\Omega$. Assume further that $G$ is a Lie group of isometries acting smoothly, freely and properly on $\Omega$, and the orbits $G(y)$ have the same Riemannian volume $|G|$ for each point $y\in\Omega$. Finally, assume that $f$ is an equivariant model satisfying $\Omega_f=G(\bD)$.
      Then the representation gap of $f$ is
      \begin{equation}
    \begin{aligned}
    \text{(i.i.d.)}\quad 
    \cR_n &\sim_{\mathbb{P}} \frac{|G| J_{d_{\Omega/G}} \cV_{d_{\Omega/G}}}{n^{2/d_{\Omega/G}}}
    \qquad\qquad
    \text{(optimal)}\quad 
    \cR_n^{\star} \sim \frac{|G| J_{d_{\Omega/G}}^* \cV_{d_{\Omega/G}}^*}{n^{2/d_{\Omega/G}}}
    \end{aligned}
    \end{equation}
      where the constants are computed with respect to the quotient metric $\ell_{\Omega/G}$ on $\Omega/G$.
\end{prop}

\begin{proof}
    The idea is to apply the Fubini theorem to factorize the integration over each orbit. We have $\Omega_f=G(\bD)$. Therefore, using the quotient decomposition of the Riemannian measure and the fact that the action of $G$ is isometric with constant-volume orbits (see for instance \cite{gallot_riemannian_1990}), we obtain
    \begin{equation*}
        \cR(\Omega,\Omega_f) = \int_\Omega \min_{z\in G(\bD)}\ell(y,z)p(y)\dy =|G|\int_{\Omega/G}\min_{z\in\bD}\ell_{\Omega/G}(y,z)p(y)\dy,
    \end{equation*}
     where $p$ now denotes the induced density on the quotient space $\Omega/G$.
    Therefore, we are in the setting of Propositions \ref{app_prop:optimal_efficiency_manifold} and \ref{app_prop:random_efficiency_manifold}, since $\Omega/G$ is a manifold and $\ell_{\Omega/G}$ is the quotient metric on $\Omega/G$ and $\dy$ is the induced Riemannian measure on $\Omega/G$. We can then conclude
    \begin{equation*}
        \cR_n \sim_\bP \frac{|G| J_{d_{\Omega/G}} \cV_{d_{\Omega/G}}}{n^{2/d_{\Omega/G}}}\;,   
      \end{equation*}
    and likewise 
    \begin{equation*}
        \cR^*_n \frac{|G| J_{d_{\Omega/G}}^* \cV_{d_{\Omega/G}}^*}{n^{2/d_{\Omega/G}}}\;.   
      \end{equation*}
\end{proof}

Proposition \ref{app_prop:efficiency_equivariant} features an asymptotic evolution similar to the general case described in Propositions \ref{app_prop:optimal_efficiency_manifold} and \ref{app_prop:random_efficiency_manifold}. In particular, we recover these formulas respectively when the group $G$ contains only the identity.

\section{Conditional tasks}\label{appsec:conditional}

\subsection{Discrete-class conditioning}

We now extend these results to the more general case of conditional tasks. Both $\Omega$ and $\bD$ are subsets of $\cX\times\cY$. Let us first focus on the case where $\Omega_\cX$ is finite and covered by the input dataset $\bD_\cX$. It is clear that for each input $x\in\bD_\cX$, the Propositions \ref{app_prop:optimal_efficiency_manifold}, \ref{app_prop:random_efficiency_manifold} and \ref{app_prop:efficiency_equivariant} apply to the conditional dataset $\bD_x$ and the conditional manifold $\Omega_x$. We summarize this observation in the following Proposition.

\begin{prop}[Representation gap for discrete conditional generation]
\label{app_prop:efficiency_discrete} 
Assume that the input space $\Omega_\cX$ is finite and that $\bD_\cX=\Omega_\cX$. Assume further that each conditional manifold $\Omega_x$ is a compact Riemannian manifold without boundary, with common dimension $d_{\Omega}$, and conditional density $p_x$ that is continuous and strictly positive on $\Omega_x$. For each $x\in\Omega_\cX$, assume that $G$ is a Lie group of isometries acting smoothly, freely and properly on $\Omega_x$, and that the orbits $G(y)$ have constant Riemannian volume $|G|$ for each point $y\in\Omega_x$. Finally, assume that $f$ is an equivariant model satisfying $\Omega_f=G(\bD)$.
Then the representation gap of $f$ satisfies
\begin{equation}\label{appeq:efficiency_dicrete_conditional}
\begin{aligned}
\text{(i.i.d.)}\quad 
\cR_n
&\sim_{\mathbb{P}}
\frac{|G|}{n^{2/d_{\Omega/G}}}
\sum_{x\in\Omega_\cX}
J_x \cV_x
\qquad\qquad
\text{(optimal)}\quad 
\cR_n^{\star}
\,\sim\,
\frac{|G|}{n^{2/d_{\Omega/G}}}
\sum_{x\in\Omega_\cX}
J_x^* \cV_x^*
\end{aligned}
\end{equation}
where $J_x$ and $\cV_x$ are computed on the quotient manifold $\Omega_x/G$ with respect to the quotient metric induced by $\ell$.
\end{prop}

\begin{proof}
Since $\Omega_\cX$ is finite, the representation gap decomposes as a finite sum over conditional manifolds:
\begin{equation*}
    \cR(\Omega,\Omega_f) = \sum_{x\in\Omega_\cX} \cR(\Omega_x,(\Omega_f)_x).
\end{equation*}
We conclude by applying Proposition \ref{app_prop:efficiency_equivariant} independently on each conditional manifold $\Omega_x$.
\end{proof}

Note that Proposition \ref{app_prop:efficiency_discrete} naturally generalizes to the setting where the conditional manifolds $\Omega_x$ have different dimensions $d_x$ for each $x\in\Omega_\cX$. In this case, the representation gap is determined by the conditional manifolds with the highest dimension. In particular, the intrinsic dimension becomes $d=\max_{x\in\Omega_\cX}d_x$.

\subsection{Continuous conditioning}

We now turn to the case where $\Omega_\cX$ is continuous. Clearly, we require some result on how $f$ behaves outside the training data $\bD_\cX$. We assume that $f$ is Lipschitz with constant $L$, which is a standard hypothesis in neural network analysis.

We focus on supervised prediction. Each input $x\in\Omega_\cX$ is associated with a unique target $y(x)\in\Omega_\cY$, so that the observation manifold $\Omega\subset\Omega_\cX\times\Omega_\cY$ can be identified with the graph of the function $y:\Omega_\cX\rightarrow\Omega_\cY$. In particular, the intrinsic dimension of $\Omega$ coincides with that of the input manifold, \textit{i.e.}, $d_\Omega=d_{\Omega_\cX}$, independently of the dimension of $\cY$. We further assume that the model $f$ generates a unique prediction $f(x)$ for each input $x\in\Omega_\cX$, so that the prediction manifold $\Omega_f$ can similarly be identified with the graph of $f$. In this context, the conditional representation gap is defined by
\begin{equation}\label{appeq:representation_gap_conditional}
    \cR(\Omega,\Omega_f)=\int_{\Omega_\cX}\min_{z'\in\Omega_f}\ell(z,z')\,p(x)\,dx,
\end{equation}
where $z=(x,y(x))\in\Omega$ and $z'=(x',f(x'))\in\Omega_f$. We denote by $\ell_\cX$ and $\ell_\cY$ the metrics induced by $\ell$ on $\cX$ and $\cY$ respectively.

We now study how to generalize the result of Proposition \ref{app_prop:efficiency_discrete} to the conditional setting with continuous conditioning. It is unclear whether we can derive a clean asymptotic equivalent of the representation gap in this case, since the geometry of $\Omega$ becomes critical due to the coupling between input and target. However, the next Proposition introduces an upper bound that follows the form introduced in Propositions
\ref{app_prop:optimal_efficiency_manifold}, \ref{app_prop:random_efficiency_manifold}, \ref{app_prop:efficiency_equivariant} and \ref{app_prop:efficiency_discrete}.

\begin{prop}[Conditional representation gap of an equivariant model]\label{app_prop:conditional_efficiency_equivariant}
    Assume that $\Omega$ is a compact $d_{\Omega}$-dimensional Riemannian manifold without boundary and $p$ is continuous and strictly positive on $\Omega$. Assume also that $G$ is a Lie group of isometries acting smoothly, freely and properly on $\Omega_\cX$, and that the orbits $G(x)$ have constant Riemannian volume $|G|$ for each point $x\in\cX$. Assume that $\ell$ is additively separable. Finally, assume that $f$ is an equivariant model satisfying $\Omega_f=G(\bD)$, which is furthermore L-lipschitz with constant $L>0$.
    Then the representation gap satisfies
    \begin{equation}\label{appeq:efficiency_conditional}
    \begin{aligned}
    \text{(i.i.d.)}\quad 
    \cR_n &= O_{\mathbb{P}} \left(\frac{1}{n^{2/d}}\right)
    \qquad\qquad
    \text{(optimal)}\quad 
    \cR_n^{\star} &= O\left(\frac{1}{n^{2/d}}\right)
    \end{aligned}\;,
    \end{equation}
    where $\Omega_\cX/G$ denotes the quotient space of $\Omega_\cX$ by the symmetry group $G$, and $d=d_{\Omega_\cX/G}$ denotes the dimension of $\Omega_\cX/G$.
\end{prop}

\begin{proof}
Using the equivariance of $f$, we proceed as in the proof of Proposition \ref{app_prop:efficiency_equivariant} and reduce the representation gap to a quantization problem on the quotient manifold $\Omega_\cX/G$. Then, let $z=(x,y(x))\in\Omega$ denote a data sample with input $x$, let $\hat x=\argmin_{x'\in\bD_\cX}\ell_\cX(x,x')$ denote the nearest training input to $x$, and let $\hat{z}=(\hat{x},y(\hat{x}))$ denote the corresponding training sample. Since $f$ interpolates the training dataset, we have $f(\hat x)=y(\hat x)$. Using the additive separability of $\ell$ and the Lipschitz continuity of $f$,
\begin{equation*}
  \ell\big((x,y(x)),(\hat x,f(\hat x))\big)
\leq
\ell_\cX(x,\hat x)
+
\ell_\cY(f(x),f(\hat x))
\leq
(1+L)\ell_\cX(x,\hat x)\;.  
\end{equation*}

Thus, $\cR(\Omega,\Omega_f)
\leq
(1+L)
\int_{\Omega_\cX}
\min_{x'\in\bD_\cX}
\ell_\cX(x,x')
\,p(x)\,dx,$
and the result follows from Proposition \ref{app_prop:efficiency_equivariant}.
\end{proof}

\section{Link with related work}\label{appsec:link_related_work}

In this section, we clarify the relations of the concept introduced in this article with several related works.

\subsection{Generalization error}\label{appsec:link_generalization_error}

A natural question is to relate the representation gap $\cR(\Omega,\Omega_f)$ to the generalization error \citep{shalev2014understanding},
commonly used to characterize generalization. We focus on the setting of prediction tasks, for which there is a widely accepted definition of the generalization error, 
$\mathcal{E}=\int_{\Omega}\ell(y(x),f(x))p(x)\dx$. 

\begin{prop}[Comparison with generalization error]\label{app_prop:generalization_error_link}
    If the model $f$ is $L$-Lipschitz and $\ell$ is additively separable, we have
    \begin{equation}\label{appeq:generalization_error}
        \frac{1}{1+L}\mathcal{E}\leq \cR(\Omega,\Omega_f)\leq\mathcal{E}\;.
    \end{equation}
\end{prop}

\begin{proof}
    In the supervised setting, the representation gap can be written
    \begin{equation*}
        \cR(\Omega,\Omega_f)=\int_\Omega\inf_{x'}\ell\left( (x,y(x))\,,\,(x',f(x'))\right)p(x)\dx.
    \end{equation*}
    We can see that $\cR(\Omega,\Omega_f)\leq\mathcal{E}$ (due to the inf operator).
    Moreover, the $L$-Lipschitz regularity of $f$ implies that deviations in the output space are controlled by deviations in the input space. Using the additivity of $\ell$ together with the Lipschitz bound, we obtain
    \begin{equation*}
        \ell_\cY(y(x),f(x))\leq(1+L)\;\inf_{x'}\ell\left((x,y(x))\,,\,(x',f(x'))\right),
    \end{equation*}
    and we conclude by integrating over $\Omega$.
\end{proof}

Combining Theorems \ref{prop:efficiency_conditional} and \ref{prop:generalization_error}, we obtain $\mathcal{E} = O\left(1/n^{2/d_{\Omega}}\right)$ as $n\rightarrow+\infty$, a result closely related to \citet{tahmasebi_exact_nodate}. Moreover, $\cR(\Omega,\Omega_f)=0$ implies $f(x)=y_x$ almost everywhere, and therefore $\mathcal{E}=0$. Generalization error and representation gap are therefore closely related. 

\subsection{Wasserstein distance}\label{appsec:link_wasserstein}

Wasserstein distance \cite{peyre2019computational} is typically used to measure neural network generalization \citep{theis2015note}. Interestingly, we can see that the representation gap $\cR(\Omega,\Omega_f)$ is a particular case of the Wasserstein distance $\mathcal{W}(\Omega,\Omega_f)$, where each point $(x,y)\in\Omega$ is associated to the probability mass $p(x,y)$ and each prediction point $z\in\Omega_f$ is associated to the mass of its Voronoi cell.  

\section{Compute resources}

All experiments were conducted on consumer-grade hardware. The main experiments were independently reproduced on Google Colab using the default free public configuration. No specialized compute infrastructure or large-scale GPU resources were required.

\section{Limitations}

Our analysis relies on the following main assumptions.

\textbf{Asymptotic regime}.
Our results are asymptotic in the dataset size $n$. However, our experiments suggest that the asymptotic regime is reached relatively quickly in practice (see Section \ref{sec:validating}).

\textbf{Model assumptions}.
Our results on generative modeling focus on DDIM diffusion models. While the analysis extends naturally to the broader class of linear Gaussian diffusion models --- including DDPM, variance-exploding, and variance-preserving formulations --- recent diffusion architectures may fall outside this framework. Moreover, our analysis assumes exact equivariance constraints and fully optimized models. These assumptions are motivated by common practices in deep learning, where models are often trained in the interpolation regime and equivariance is enforced directly at the architectural level. Nonetheless, extending the theory to approximate equivariance or partially trained models would nevertheless be of significant interest.

\textbf{Geometric assumptions}.
Our theory relies on several regularity assumptions on the data manifold $\Omega$ and the symmetry group $G$, including compactness and smoothness of the group action. While such assumptions are standard in geometric learning theory, real-world datasets may only approximately satisfy them. Nevertheless, our empirical results on real-world datasets suggest that the theory remains informative beyond the idealized setting considered here.

Furthermore, some assumptions could likely be relaxed without fundamentally changing the analysis. For instance, the constant-volume orbit assumption is mainly introduced to simplify the exposition.

\textbf{Supervised prediction}.
In the setting of supervised prediction, we derive asymptotic bounds rather than precise asymptotic equivalents. Unlike the generative setting, the geometry of the joint manifold $\Omega$ becomes critical due to the coupling between inputs and targets, making a full asymptotic characterization more challenging. A deeper study of this regime is an important direction for future work.

\section{Broader impact}

This work is primarily theoretical and aims to improve the understanding of generalization, sample efficiency, and equivariance in modern machine learning systems. By relating generalization behavior to the intrinsic geometry of data manifolds, our results may contribute to the development of more data-efficient learning algorithms and better principled model design.

In particular, improved sample efficiency could benefit applications where data collection is expensive or limited, such as scientific imaging, healthcare, or robotics. More broadly, our analysis may help clarify the role of symmetries and geometric structure in deep learning systems.

At the same time, advances in generative modeling and sample-efficient learning may contribute to the development of more capable generative systems, including systems that could potentially be misused for synthetic media generation or large-scale content production. However, the present work does not introduce new generative architectures or deployment methods, and focuses instead on the theoretical understanding of existing approaches.


\end{document}